%% file: sirocco2026-main.tex
\documentclass[11pt,a4paper]{article}

\usepackage[T1]{fontenc}
\usepackage{graphicx}
\usepackage{latexsym}
\usepackage{multirow}
\usepackage{algorithm}
\usepackage{algpseudocode}
\usepackage{amssymb,amsthm}
\usepackage{amsfonts}
\usepackage{amsmath}
\usepackage{txfonts}
\usepackage{multirow,eepic}
\usepackage{wrapfig}
\usepackage{xspace}
\usepackage{lipsum}
\usepackage{pdfpages}
\usepackage{tikz}
\usepackage{enumitem}
\setlist[itemize]{itemsep=0pt, parsep=0pt, topsep=2pt, partopsep=0pt}

\usepackage{tikz}
\usepackage{subcaption} % 2枚を並べる場合
\usepackage{booktabs}
\usetikzlibrary{positioning}
\usetikzlibrary{automata}
\usetikzlibrary{arrows}
\usetikzlibrary{calc}
\usetikzlibrary{arrows.meta}
\usepackage{hhline}
\makeatletter
\def\senbun#1(#2)#3({\@senbun(#2)(}
\def\@senbun(#1,#2)(#3,#4){%
   \@tempdima#1\p@ \advance\@tempdima#3\p@
   \divide\@tempdima\tw@
   \@tempdimb#2\p@ \advance\@tempdimb#4\p@
   \divide\@tempdimb\tw@
   \edef\@senbun@temp{\noexpand\qbezier(#1,#2)%
      (\strip@pt\@tempdima,\strip@pt\@tempdimb)(#3,#4)}%
  \@senbun@temp}
\makeatother

\usepackage{graphicx}

\usepackage{multirow}

\usepackage{float}
\usepackage{xcolor}

\usepackage{authblk}
\usepackage{geometry}
\usepackage{amsmath,amssymb,amsthm}
\usepackage{graphicx}
\usepackage{hyperref}

\newtheorem{lemma}{Lemma}
\newtheorem{theorem}{Theorem}

\theoremstyle{definition}

 \newtheorem{definition}{Definition}

\newcommand{\ASY}{{\sc Asynch}}
\newcommand{\FSY}{{\sc Fsynch}}
\newcommand{\SSY}{{\sc Ssynch}}
\newcommand{\RSY}{{\sc Rsynch}}

\newcommand{\RR}{\textsc{RR}\xspace}

\newcommand{\LU}{\ensuremath{\mathcal{LUMI}}\xspace}
\newcommand{\FS}{\ensuremath{\mathcal{FST\!A}}\xspace}
\newcommand{\FC}{\ensuremath{\mathcal{FCOM}}\xspace}
\newcommand{\OB}{\ensuremath{\mathcal{OBLOT}}\xspace}
\newcommand{\Look}{\ensuremath{\mathit{Look}}\xspace}
\newcommand{\Compute}{\ensuremath{\mathit{Compute}}\xspace}
\newcommand{\Move}{\ensuremath{\mathit{Move}}\xspace}
\newcommand{\LCM}{\ensuremath{\mathit{LCM}}\xspace}
\newcommand{\clight}{\textit{Light}}
\newcommand{\N}{{\rm I\kern-.22em N}} 
\newcommand{\Z}{{\sf Z\kern-.42em Z}} 
\newcommand{\R}{{\rm I\kern-.22em R}}

\newcommand{\LK}{{\mathit{Look}}}

\newcommand{\M}{{\mathit{Move}}}

\newcommand{\LC}{{\mathit{LC}}}

\newcommand{\CM}{{\mathit{CM}}}

\newcommand{\MCv}{{\tt MCv}}
\newcommand{\DMSD}{{\tt DMSD}}

\newcommand{\INIT}{{\tt INIT}}
\newcommand{\ANCHOR}{{\tt ANCHOR}}
\newcommand{\LUM}{{\mathcal{LUMI}}} 
\newcounter{Codeline}
\newcommand{\Newcodeline}{\setcounter{Codeline}{1}}
\newcommand{\Cl}{{\theCodeline}: \addtocounter{Codeline}{1}}
\newcommand{\crm}{\\}

\newcommand{\length}{\mathit{length}\xspace}
\newcommand{\calC}{\mathcal{C}}

\newcommand{\calR}{\mathcal{R}}
\newcommand{\calA}{\mathcal{A}}

\newcommand{\exc}{\textsf{exc}}
\newcommand{\cpy}{\textsf{cpy}}
\newcommand{\rst}{\textsf{rst}}
\newcommand{\phase}{\texttt{phase}}
\newcommand{\Stop}{\texttt{Stop}}

\newcommand{\MoveSet}[1]{\mathrm{MoveSet}(#1)}
\newcommand{\moves}[1]{\mathrm{moves}(#1)}
\newcommand{\Stopped}[1]{\mathrm{Stopped}(#1)}

\newcommand{\Dyad}{\mathbb{Z}[1/2]}

\newtheorem{lemmaduplicate}{Lemma}

\title{Two-Robot Computational Landscape: A Complete Characterization of Model Power in Minimal Mobile Robot Systems\thanks{{
This work was supported by JSPS KAKENHI Grant Numbers
JP20KK0232, JP23K16838, %(Kitamura)
JP25K03078, %(Sudo)
JP25K03079, %(Wada)
and by JST FOREST Program JPMJFR226U.}}

} %TODO Please add
\date{}
\author[1]{Naoki Kitamura}
\author[2]{Yuichi Sudo}
\author[2]{Koichi Wada\thanks{Corresponding Author: wada@hosei.ac.jp}}

\affil[1]{The University of Osaka, Japan}
\affil[2]{Hosei University, Japan}

\begin{document}

\maketitle


\begin{abstract}
The computational power of autonomous mobile robots under the Look--Compute--Move (LCM) model
has been widely studied through an extensive hierarchy of robot models defined by the presence
of memory, communication, and synchrony assumptions.
While the general $n$-robot landscape has been largely established, the exact structure for
\emph{two robots} has remained unresolved.
This paper presents the first complete characterization of computational power for two autonomous robots
across all major models---$\OB$, $\FS$, $\FC$, and $\LU$---under the full spectrum of schedulers
(\FSY, \RSY, \SSY, \ASY, and their atomic variants).

Our results reveal a landscape that fundamentally differs from the general case.
Most notably, we prove that $\FS$ and $\LU$ coincide under full synchrony, a surprising collapse
indicating that perfect synchrony can substitute both memory and communication when only two robots exist.
We also show that both the color complexity and the computation time of the simulator proving the equivalence can be substantially improved.
Moreover, We show that $\FS$ and $\FC$ are orthogonal:
there exists a problem solvable in the weakest communication model but impossible even in the strongest
finite-state model, completing the bidirectional incomparability.
%All equivalence and separation results are derived through a novel simulation-free method,
%providing a unified and constructive view of the two-robot hierarchy.
This yields the first complete and exact \emph{computational landscape} for two robots,
highlighting the intrinsic challenges of coordination at the minimal scale.
\end{abstract}


%\newpage
\renewcommand{\baselinestretch}{0.85}
\section{Introduction}
\label{sec:introduction}
\subsection{Background and Motivation}
The computational power of autonomous mobile robots operating through $\Look$-$\Compute$-$\Move$ ($\LCM$) cycles has been a central topic in distributed computing. Robots are modeled as anonymous, uniform, disoriented points in the Euclidean plane. In each cycle, a robot observes the configuration ($\Look$), computes a destination ($\Compute$), and moves accordingly ($\Move$). These agents can solve distributed tasks collectively, despite their simplicity.

The weakest standard model, $\OB$, assumes robots are oblivious (no persistent memory) and silent (no explicit communication). Since its introduction~\cite{SY}, this model has been extensively studied for basic coordination tasks such as Gathering~\cite{AP,AOSY,BDT,CDN,CFPS,CP,FPSW05,ISKIDWY,SY}, Pattern Formation~\cite{FPSW08,FYOKY,SY,YS,YUKY}, and Flocking~\cite{CG,GP,SIW}. The limitations imposed by obliviousness and silence have led to extensive investigation of stronger models.
%Gathering (e.g.~\cite{AP,AOSY,BDT,CDN,CFPS,CP,FPSW05,ISKIDWY,SY}),  Pattern Formation (e.g.~\cite{FPSW08,FYOKY,SY,YS,YUKY}), Flocking (e.g.~\cite{CG,GP,SIW}). 
A widely studied enhancement is the luminous robot model, $\LU$~\cite{DFPSY}, where robots are equipped with a visible, persistent light (i.e., a constant-size memory) used for communication and state retention. The lights can be updated during $\Compute$ and are visible to others, enabling both memory and communication in each cycle.

To understand the individual roles of memory and communication, two submodels of $\LU$ have been introduced: $\FS$ (persistent memory without communication), and $\FC$ (communication without memory)~\cite{apdcm,BFKPSW22,OWD}. Studying these four models—$\OB$, $\FS$, $\FC$, and $\LU$—clarifies the contribution of internal robot capabilities to problem solvability.

External assumptions such as synchrony and activation schedulers also play a crucial role. The \emph{semi-synchronous} (\SSY) model~\cite{SY} activates an arbitrary non-empty subset of robots in each round, who perform one atomic $\LCM$ cycle. Special cases include the \emph{fully-synchronous} (\FSY) model, where all robots are activated every round, and the \emph{restricted synchronous} (\RSY) model~\cite{BFKPSW22}, where following \FSY, successive activations are such that the sets of robots activated in consecutive rounds are disjoint.

In contrast, the \emph{asynchronous} (\ASY) model~\cite{FPSW99} introduces arbitrary but finite delays between phases of each robot's cycle, under a fair adversarial scheduler. This model captures high uncertainty and has proven computationally weaker.

Stronger variants of \ASY\ have been studied by assuming partial atomicity: for example, \textbf{$LC$-atomic} and \textbf{$CM$-atomic} models~\cite{OWD}. These restrict delays to only parts of the cycle and allow finer-grained modeling.

Previous studies have identified 14 distinct robot models by proving \emph{separations} (i.e., problems solvable in one model but not in another) between pairs of internal models and external schedulers, and by establishing \emph{equivalences} or \emph{orthogonality} among some of them~\cite{DFPSY,FSW19,BFKPSW22,FSSW23}.

%These internal models and external schedulers define a two-dimensional space of 14 canonical robot models. Previous work has explored their pairwise computational relationships by proving \emph{separation} (i.e., a problem solvable in one model but not the other) and identifying \emph{equivalence} or \emph{orthogonality} between some models~\cite{DFPSY,BFKPSW22,FSSW23}.

\begin{figure}[H]
    \centering
    \includegraphics[width=0.60\linewidth]{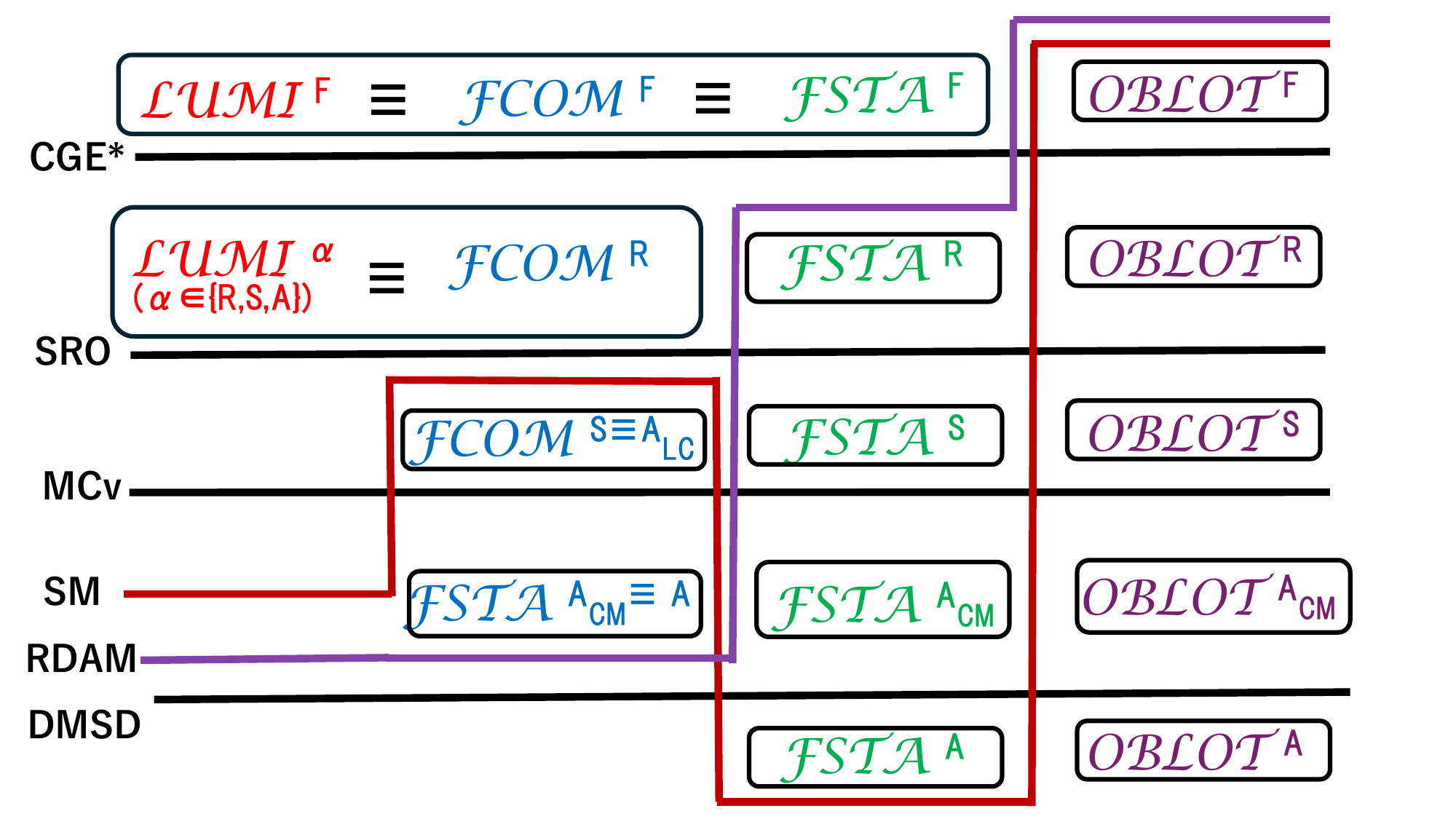}
%    \caption{Landscape of 13 robot models combining robot capabilities and scheduler classes.    }
\caption{
Computational landscape of two-robot models under major scheduler classes.
Models enclosed by the same ellipse are \emph{equivalent} in computational power.
Each connecting line represents a \emph{separating problem}:
the task above the line is solvable in the stronger model,
whereas the one below it is not.
}

    \label{fig:boundary-line}
\end{figure}
\vspace{-0.5cm}

\subsection{Contributions}

While prior research has mainly focused on multi-robot systems, the case of \emph{two robots} remains surprisingly subtle.  
With only two anonymous, identical entities, symmetry becomes absolute and geometric memory collapses: the configuration reduces to a single line segment, and no deterministic rule can break mirror symmetry without additional assumptions.  
Unlike in the multi-robot case—where light visibility ($\LU, \FC, \FS$) and scheduler type (\FSY, \RSY, \SSY, \ASY) yield a rich computational hierarchy—the two-robot landscape shows a partial collapse: under full synchrony, all robot models with light  become equivalent, yet under weaker schedulers the structure may remain as intricate as in the multi-robot setting, making its precise characterization considerably more difficult.

%Moreover, light visibility ($\LU, \FC, \FS$) and scheduler type (\FSY, \RSY, \SSY, \ASY) drastically alter the computational power
%, making the two-robot landscape qualitatively distinct from the multi-robot one.

\medskip

This paper provides the \textbf{complete solvability landscape} for all combinations of  
$\OB, \FS,\\ \FC,$ and $\LU$ under the major scheduler classes \FSY, \RSY, \SSY, \ASY, and atomic variants of \ASY.  
We classify every pair of models as either equivalent or strictly separated, thereby resolving all previously open cases for two-robot systems. Fig.~\ref{fig:boundary-line} shows the computational landscape of two robots~\footnote{In the case of two robots, there are 13 robot models, because $\LU^F (\equiv \FC^F)$ and $\FS^F$ are equivalent; in the general multi-robot case, however, these two models are separated, yielding 14 robot models~\cite{FSW19}}.

%While prior research has mainly focused on multi-robot systems, the two-robot case remains surprisingly subtle.
%$With only two entities, symmetry is maximized and multiplicity detection is meaningless, making many standard impossibility arguments fail or require new techniques.
%Moreover, in this minimal setting, scheduler assumptions (fully-synchronous, semi-synchronous, asynchronous, or partially atomic variants) have disproportionate effects on computational power.

%Previous separation studies~\cite{BFKPSW22,DFPSY,FSSW23,FSW19} established rich landscapes for $n$-robot systems, but the exact structure for \emph{two robots} was still unknown.
%This paper fills that gap by providing the \emph{complete solvability landscape} for all combinations of
%$\OB, \FS, \FC,$ and $\LU$
%with the major scheduler classes
%\FSY, \SSY, \ASY\ and variants of \ASY.
{\small
\begin{table}[ht]
  \caption{The previous results of simulations and this paper's improvement.}
  \begin{center}
    \textbf{previous results}\\[3pt]
    \begin{tabular}{|c|c|c|c|c|}
      \hline
      simulating model & simulated model & $\#$ colors & $\#$ epochs & Ref. \\
      \hhline{|=====|}
      $\LU^{A}$   & $\LU^{R}$ 
        & $3k$ & $6T$ & \cite{NTW24} \\ \hline
              $\mathcal{\FC}^{F}$ & $\LU^{F}$ 
        & $2k^2$ & $2T$ & \cite{FSW19} \\ \hline
      $\FC^{R}$ & $\LU^{S}$ 
        & $64k^2$ & $8T$ & \cite{BFKPSW22} \\ \hline
      $\LU^{S}$ & $\LUM^{R}$ 
        & $4k$ & $7T$ &  \cite{NTW24}\\ \hline
      %\hhline{|=====|}
      $\FC^{A}$ & $\FC^{A_{CM}}$ 
        & $36k$ & $8T$ & \cite{FSSW23} \\ \hline
              $\FC^{A_{LC}}$ & $\FC^{S}$ 
        & $36k$  & $8T$ & \cite{FSSW23} \\ \hline
%      \hhline{|=====|}
    \end{tabular}

        \textbf{our improvement}\\[3pt]
    \begin{tabular}{|c|c|c|c|c|}
      \hline
      simulating model & simulated model & $\#$ colors & $\#$ epochs & Ref. \\
      \hhline{|=====|}
        $\FC^{F}$ ($\FS^{F}$)  & $\LU^{F}$ 
        & $1$ & $T$ & this paper \\ \hline    
      $\FC^{R}$   & $\LU^{R}$ 
        & $3k^2$ & $3T$ & this paper \\ \hline
              $\FC^{A}$ & $\FC^{A_{CM}}$ 
        & $7k$ & $4T$ & this paper \\ \hline
              $\FC^{A_{LC}}$ & $\FC^{S}$ 
        & $7k$  & $4T$ & this paper \\ \hline
    \end{tabular}
  \end{center}

  \label{tab:sim-results}
\end{table}
}
%\subsection{Contributions}
 The contributions of this paper are as follows.
\begin{itemize}
  \item \textbf{Complete solvability map:}
        We determine all computational relationships among the 13 canonical models for two robots, including all strict separations, equivalences, and orthogonal (incomparable) pairs.

  \item \textbf{Simulator efficiency:}
        We evaluate the efficiency of simulators that realize stronger models within weaker ones, in terms of the number of colors required and the number of epochs needed to complete the simulated algorithm.
        This quantitative analysis complements the qualitative solvability map, providing a finer-grained understanding of model power. Table~\ref{tab:sim-results} shows the previous results and this paper's improvement, where the simulated algorithm has $k$ colors and finishes within $T$ epochs\footnote{These results are obtained by restricting the simulation framework to two robots. 
Moreover, the simulation of $\LU^R$ by $\FC^R$ is constructed by combining the simulation of $\LU^S$ by $\FC^R$ with that of $\LU^R$ by $\LU^S$, requiring $4\times 64k^2=256k^2$ colors and $8\times7=56$ epochs.}. By restricting the setting to two robots, both the number of colors and the number of epochs can be significantly improved.
  \item \textbf{Separation between $CM$-{\bf atomic} \ASY\ and general \ASY:}
  We prove that the $CM$-{\bf atomic} \ASY, where each robot executes its consecutive $\Compute$ and $\M$ operations atomically, is strictly stronger than the general \ASY.
  %We newly prove that $CM$-{\bf atomic} \ASY\,where each robot performs its consecutive $\Compute$ and $\M$ operations atomically, is strictly stronger than general \ASY.
      No previous work has established such a separation for two robots.
      To achieve this, we introduce a new problem specifically designed to distinguish these two models.
      This result shows that atomicity in the \Look–\Compute–\Move\ cycle can strictly enhance the computational power of asynchronous robots, even under fair activation.

        %We newly prove that Move-{\bf atomic} \ASY\ is strictly stronger than general \ASY.
        
        %This separation clarifies that atomicity constraints in the \Look–\Compute–\Move\ cycle fundamentally limit the computational power of asynchronous robots, even when fairness is guaranteed.
        \item \textbf{Collapse of $\FS^F$ and $\LU^F$, and simulation-free direct construction:}
While it was already known that $\FC^F$ and $\LU^F$ coincide for general $n$-robot systems~\cite{BFKPSW22}, our result further shows that, when only two robots exist, $\FS^F$ and $\LU^F$ also coincide under full synchrony.
This means that, in the two-robot setting, full synchrony can completely substitute both memory and communication—a surprising and conceptually important finding.
Moreover, the equivalence between $\FC$ ($\FS$) and $\LU$ under \FSY\ is \emph{constructive}:
an algorithm for $\FC$ or $\FS$ can be directly obtained from the corresponding $\LU$ algorithm without increasing the number of colors or computation time.
%  \item \textbf{Collapse of $\FS^F$ and $\LU^F$, and simulation-free direct construction:}
%Contrary to the general $n$-robot hierarchy, under full synchrony these two models coincide\footnote{$\FC^F$ and $\LU^F$ are equal for general $n$-robots case~\cite{BFKPSW22}}.
%This collapse shows that full synchrony can fully substitute both memory and communication when only two robots exist---a surprising and conceptually important finding.
%Furthermore, the equivalence between $\FC$ ($\FS$) and $\LU$ under \FSY\ is \emph{constructive}: 
%an algorithm for $\FC$ or $\FS$ can be directly obtained from the corresponding $\LU$ algorithm without increasing the number of colors or computation time.
%Finally, our equivalence and separation results are established through a novel \emph{direct-construction} approach without relying on inter-model simulation, yielding cleaner proofs and deeper insights into the internal structure of the hierarchy.

  %\textbf{Collapse of $\FS^F$ and $\LU^F$, and simulation-free direct proofs:}
%        Contrary to the general $n$-robot hierarchy, under full synchrony these two models coincide.
%        This collapse shows that full synchrony can fully substitute both memory and communication when only two robots exist---a surprising and conceptually important finding.
%        Moreover, our equivalence and separation results are established through a novel \emph{direct-construction} approach without relying on inter-model simulation, yielding cleaner proofs and deeper insights into the internal structure of the hierarchy.

  \item \textbf{Orthogonality of $\FC$ and $\FS$:}
        We newly identify a problem solvable in the weakest communication model but impossible even in the strongest finite-state model, complementing the previously known reverse direction.
        This establishes a true \emph{orthogonality} between memory and communication capabilities for two robots.

  \item \textbf{Intrinsic difficulty of two-robot systems:}
        The two-robot case accentuates the role of disorientation, activation symmetry, and move atomicity.
        Tasks that are trivial for three or more robots (e.g., rotational expansion or convergence) become nontrivial or even unsolvable, revealing the delicate interplay between synchrony and minimal multiplicity.
\end{itemize}

These findings provide the first complete and exact landscape of computational power for two autonomous robots---closing a long-standing gap in the theory of mobile robot computation and establishing new foundations for understanding minimal distributed systems.

%\subsection{Related Work}

%Our study builds on and extends the general multi-robot hierarchy introduced in
%\textit{Beyond Pairwise Comparisons: Unveiling Structural Landscape of Mobile Robot Models}~\cite{FSW19}
%and on discrete-space analogs such as computational landscapes over graphs and grids %(e.g.,~\cite{discreteLandscape1,discreteLandscape2}).
%By isolating the two-robot case, we expose unique algebraic and topological properties invisible in higher dimensions, contributing to a unified understanding of distributed computation in both continuous and discrete spaces.

\section{Preliminaries}
\label{sec:preliminaries} \subsection{Robots}
%-------------------------------------
%\textcolor{red}{Note that strong multiplicity detection is not assumed}

We consider a set $R_n = \{ r_0, \ldots, r_{n-1} \}$ of $n > 1$ mobile computational entities, called {\em robots}, which move and operate in the Euclidean plane $\mathbb{R}^2$, where they are viewed as points. In this paper, we consider the case that $n=2$ and the two robots are denoted $a$ and $b$. 
We simply write $R$ instead of $R_2$.%, except when it is necessary to emphasize the number of robots.
The robots are {\em autonomous} (i.e., they operate without central control or external supervision),  
{\em identical} (i.e., they are indistinguishable by appearance and do not have unique identifiers), and  
{\em homogeneous} (i.e., they execute the same program).  
Throughout this paper, we denote the position of a robot $r \in R$ at time $t \in \mathbb{R}_{\ge 0}$ by $r(t) \in \mathbb{R}^2$.

Each robot is equipped with a local coordinate system (in which it is always at its origin), and it is able
to observe the positions of the other robots in its local coordinate system.
The robots are {\em variable disoriented} (i.e., that is, there might not be consistency between the coordinate systems of different robots at the same time, or of the same robot at different times \footnote{The disorientation is said to be {\em self-consistency} if it is assumed that each local coordinate system always remains the same.}).
The robots however have {\em chirality}; that is, they agree on the same circular orientation of the plane (e.g., ``clockwise'' direction). 

A robot is endowed with motorial and computational capabilities. When active, a robot performs a $\LK$-$\Compute$-$\M{}$ ($\LCM$) cycle:
\begin{enumerate}
\item $\LK$: The robot obtains an instantaneous snapshot of the positions occupied by the robots, expressed in its local coordinate system. It cannot detect whether multiple robots are at the same location, nor how many (i.e., there is no {\em multiplicity detection}).
%\textcolor{red}{[Koichi: We must have assumed no multiple detection.]} 
\item $\Compute$: The robot executes its built-in deterministic algorithm, identical for all robots, using the snapshot as input. The result of the computation is a destination point.
\item $\M$: The robot moves continuously in a straight line until it reaches the computed destination (i.e., movements are {\em rigid}).\footnote{
Movements are said to be \emph{non-rigid} if they may be interrupted by the adversary. 
}
If the destination is the current location, the robot stays still.
\end{enumerate}

In the standard model, $\OB$, the robots are also {\em silent}: 
they have no means of direct
communication of information to other robots; 
furthermore, they are {\em oblivious}: at the start of a cycle, a robot has no
memory of observations and computations performed in previous cycles.

In the other common model, $\LUM$,
each robot $r$ is equipped with a persistent register $\clight[r]$, called a {\em light},
whose value—referred to as its {\em color}—is drawn from a constant-sized set $C$ and is visible to all robots.
The color can be set by $r$ at the $\Compute$ operation, and it is not automatically reset at the end of a cycle.
In $\LUM$, the $\LK$ operation returns a colored snapshot; that is, it returns the set of distinct pairs $(\textit{position}, \textit{color})$
representing the states of the other robots at that time.
Note that if $|C|=1$, this case corresponds to the $\OB$ model.

%It is sometimes convenient to describe a robot $r$ as having $k \geq 1$ lights, denoted $r.\light_1, \ldots, r.\light_k$,
%where each $r.\light_i$ takes values from a finite color set $C_i$.
%In this view, $\clight[r]$ is treated as a $k$-tuple of variables.
%Clearly, this is equivalent to $r$ having a single light with color set $C = C_1 \times C_2 \times \dots \times C_k$.
%As we will see later, in the $\FC$ model, this division of lights is not merely for convenience—it may fundamentally help the robots solve some problems.

 %\footnote{If  (strong) multiplicity detection is assumed, the snapshot is a multi-set.}. 

 Two submodels of $\LU$ have been defined and investigated, $\FS$ and $\FC$, each offering only one
 of its two capabilities, persistent memory and 
direct means of communication, respectively.  
In $\FS$, a robot can only see the color of its own lights; thus, the color merely encodes an internal state. Therefore, robots are {\em silent}, as in $\OB$, but they are {\em finite-state}. 
In $\FC$, a robot can only see the colors of the lights of the other robots;
thus, it can communicate the color of its own lights to others,
but it cannot remember its own state (i.e., its own $k$ colors).
Hence, robots are enabled with {\em  finite-communication}  but are {\em oblivious}. 

%The $\FC$ model requires more explanation. In this model, each robot $r$ cannot observe its own $k$ lights during the $\LK$ operation,  and must determine the destination point and the next color of each light during the $\CP$ operation based solely on the colored snapshot,  
%which does not include its own $k$ lights.  
%At this point, the robot $r$ selects one of the following two options for each $\light_i$:  
%(i) keep the current color of $\light_i$, or  
%(ii) change the color of $\light_i$ to $c_i$, where $c_i$ is determined solely based on the snapshot excluding $r$'s lights.

%-------------------------------------
\subsection{Schedulers,  Events}
%-------------------------------------

In this subsection, we define several schedulers---\FSY, \RSY, \SSY, \ASY, 
and several subclasses of \ASY
%$LC$-atomic~\ASY, $M$-atomic~\ASY, and $CM$-atomic~\ASY
---in terms of the activation schedule of the robots and the duration of their $\LCM$ cycles.  
These schedulers impose different constraints on the adversary, as we shall see,  
but they all share a common constraint called \emph{fairness}:  
every robot must perform the $\LK$, $\Compute$, and $\M$ operations infinitely often.

The time it takes to complete a cycle is assumed to be finite and the operations $\LK$
 and  $\Compute$  are  assumed to be instantaneous. 

We assume that each of the $\LK$, $\Compute$, and $\M$ operations completes in finite time.  
Moreover, the $\LK$ and $\Compute$ operations are assumed to be instantaneous.  
In the literature, the $\Compute$ operation has also been modeled as having some nonzero duration,  
at the end of which a robot changes its color (in the $\LU$ and $\FC$ models).  
However, we can assume without loss of generality that the $\Compute$ operation is instantaneous:  
since the robot’s color always changes exactly at the end of the $\Compute$ operation,  
its duration has no effect on the subsequent execution.

For any robot $r \in R$ and $i \in \mathbb{Z}_{> 0}$, let  
$t_L(r,i)$, $t_C(r,i)$, $t_B(r,i)$, and $t_E(r,i) \in \mathbb{R}_{\ge 0}$  
denote the times at which $r$ performs its $i$-th Look, Compute, Move-begin, and Move-end operations, respectively.  
These satisfy:  
$t_L(r,i) < t_C(r,i) < t_B(r,i) < t_E(r,i) < t_L(r,i+1)$.  
If $r$ decides to move from $p_B$ to $p_E$ at $t_C(r,i)$, it moves continuously along $[p_B, p_E]$ during $[t_B(r,i), t_E(r,i)]$, with $r(t_B(r,i)) = p_B$, $r(t_E(r,i)) = p_E$, and variable speed: for any $t_1 < t_2$ in $[t_B(r,i), t_E(r,i)]$,  
$|p_B - r(t_1)| < |p_E - r(t_2)|$ if $p_B \ne p_E$.

In the $\LU$ and $\FC$ models, if $r$ changes its light from $c_1$ to $c_2$ at time $t$ and another robot $s$ performs Compute at $t$, then $s$ observes $c_2$.

%For any robot $r \in R$ and positive integer $i \in \mathbb{Z}_{> 0}$,  
%we denote by $t_L(r,i)$, $t_C(r,i)$, $t_B(r,i)$, and $t_E(r,i) \in \mathbb{R}_{\ge 0}$  
%the times at which robot $r$ performs its $i$-th $\LK$ operation, $\CP$ operation,  
%begins its $i$-th $\M$ operation, and ends its $i$-th $\M$ operation, respectively.  
%For any $r \in R$ and $i \in \mathbb{Z}_{> 0}$, it must hold that  
%$t_L(r,i) < t_C(r,i) < t_B(r,i) < t_E(r,i) < t_L(r,i+1)$.  
%If robot $r$ decides to move from a point $p_B \in \mathbb{R}^2$ to a point $p_E \in \mathbb{R}^2$  
%at time $t_C(r,i)$, it moves along the line segment $[p_B, p_E]$ during the time interval $[t_B(r,i), t_E(r,i)]$,  
%where $r(t_B(r,i)) = p_B$ and $r(t_E(r,i)) = p_E$.  
%The robot’s speed is arbitrary and may vary over time, but its movement is continuous; that is, for any $t_B(r,i) \le t_1 < t_2 \le t_E(r,i)$, $|p_B - r(t_1)| < |p_E - r(t_2)|$ holds unless $p_B = p_E$.
%In the $\LU$ and $\FC$ models, if a robot $r$ changes its color from $c_1$ to $c_2$ at time $t$ and another robot $s$ simultaneously performs $\CP$ operation, $s$ observes the new color $c_2$ as the color of $r$.

%In the $\LU$ and $\FC$ models, if a robot $r$ changes its color from $c_1$ to $c_2$, and another robot $s$ simultaneously performs a $\CP$ operation, then $s$ observes the new color $c_2$ as the color of $r$.

%For simplicity, we define $t_E(r,0)=-1$ for any $r \in R$.

%First, we define the most standard two schedulers---$\SSY$ and $\ASY$. 

In the the {\em synchronous} setting (\SSY), also called {\em semi-synchronous} and first studied in \cite{SY},  time is divided into discrete
intervals, called {\em rounds}; in each round, a non-empty set of robots is activated and they simultaneously perform
a single  $\LK$-$\Compute$-$\M$ cycle in perfect synchronization. 
The particular  synchronous setting, where every robot is activated in every round
is called {\em fully-synchronous} (\FSY).

We define \RSY\ as the semi-synchronous scheduler in which,
after an optional finite prefix of fully synchronous rounds,
the remaining rounds activate non-empty subsets of robots 
with the constraint that any two consecutive subsets are disjoint. In the case of $n=2$, the phase is a (prefix of) schedule in round-robin (\RR) where $a$ and $b$ appear alternately. 

%  Consider now  the synchronous  scheduler, we shall call \RSY,  obtained from \SSY\ by
%adding the following {\em restricted-repetition
%  condition} to its activation sequences:
%  Given a synchronous scheduler $\cal S$ and a  set of robots 
%$R$, an {\em activation  sequence}  of $R$ by (or, under) $\cal S$
%  is  an infinite sequence  $E= \langle e_1, e_2, \ldots, e_i, \ldots \rangle$,
%where   $e_i\subseteq R$ denotes the set of robots
%activated in round $i$,
%  {\small
%\begin{multline*}
%\bigg[\forall i\geq 1, e_i=R\bigg] \textbf{ or } \bigg[\exists p\geq 0 : \bigg(\big[\forall  %i\leq p,  (e_i=R)\big] \textbf{ and }\\ \big[\forall  i > p,(e_i \neq \emptyset \textbf{ and } e_i \neq R \textbf{ and }  e_i \cap e_{i+1} = \emptyset)\big]\bigg)\bigg],
%\end{multline*}
%}%
%where an {\em activation  sequence}  of $R$ %by (or, under) $\cal S$
%  is  an infinite sequence  $E= \langle e_1, e_2, \ldots, e_i, \ldots \rangle$, and
%$e_i\subseteq R$ denotes the set of robots
%activated in round $i$.
%That is,  this scheduler is composed of sequences where
% the prefix is a  (possibly empty) sequence of $R$ and, if the prefix is finite, the rest are non-empty sets
%  satisfying the constraint $(e_i \cap e_{i+1} = \emptyset)$.

In the {\em asynchronous} setting (\ASY), first studied in \cite{FPSW99}, we do not have any assumption---except for the fairness mentioned before---on the timing of each $\LK$, $\Compute$, and $\M$ operation and the duration of each $\M$ operation. 

%\red{
% Let ${\cal T} =\{\T_1, \T_2, ...\}$ denote the infinite ordered set of %all relevant times; i.e.,
%$\T_i < \T_{i+1}, i\in \N$.  %In the following, to simplify the presentation and without any loss of generality, we will refer to $\T_i\in {\cal T}$  simply by its index $i$; i.e.,  the expression ``time $\tau$'' will be used to mean
%``time $\T_{\tau}$.'' [Yuich: I think it might be clearer for readers if we can also consider real times for the timing of those events, instead of discrete times, because we define \emph{problems} according to real times. So, if you agree, I delete this paragraph. What do you think? 
%]
%}

In the rest of this subsection, we introduce subclasses of \ASY,  
classified according to the level of atomicity of the $\LK$, $\Compute$, and $\M$ operations.

\begin{itemize}
\item $\LC$-{\bf atomic}-\ASY:  
Under this scheduler, each robot performs its consecutive $\LK$ and $\Compute$ operations atomically;  
that is, no robot obtains a snapshot during the interval between the $\LK$ and $\Compute$ operations of any other robot~\cite{DKKOW19,OWD}.  
Formally, this scheduler satisfies the following condition:  
$$ \forall r, s \in R,\ \forall i, j \in \mathbb{Z}_{> 0}: \ t_L(r,i) \notin (t_L(s,j), t_C(s,j)]. $$
Note that the interval $(t_L(s,j), t_C(s,j)]$ is left-open; hence, this scheduler may allow two or more robots to obtain snapshots simultaneously.
\item \textbf{$M$-atomic-\ASY}:  
Under this scheduler, each robot performs its $\M$ operation atomically; that is, no robot obtains a snapshot while any other robot is performing its $\M$ operation~\cite{DKKOW19,OWD}.  
Formally, this scheduler satisfies the following condition:  
$$ \forall r, s \in R,\ \forall i, j \in \mathbb{Z}_{> 0}: \ t_L(r,i) \notin [t_B(s,j), t_E(s,j)]. $$

\item $\CM$-{\bf atomic}-\ASY:  
Under this scheduler, each robot performs its consecutive $\Compute$ and $\M$ operations atomically;  
that is, no robot obtains a snapshot during the interval between the $\Compute$ and $\M$ operations of any other robot. 
%including the exact moment when the other robot performs its $\CP$ operation.  
Formally, this scheduler satisfies the following condition:  
$$ \forall r, s \in R,\ \forall i, j \in \mathbb{Z}_{> 0}: \ t_L(r,i) \notin [t_C(s,j), t_E(s,j)]. $$
In the $\FS$ and $\OB$ models, a robot cannot observe the peer’s light between its Compute and Move-Begin phases; therefore, $M$-atomic-\ASY\ and  $\CM$-{\bf atomic}-\ASY\ are equivalent trivially.
In contrast, in the $\FC$ models, it has been shown that $M$-atomic-\ASY\ and  $\CM$-{\bf atomic}-\ASY\ are equivalent~\cite{FSSW23}. 
%In contrast, in the $\LU$ and $\FC$ models, a robot can observe the peer’s light, but what it sees between Compute and Move-Begin is identical to what it would see at Move-Begin.
%Hence, for $M \in \{\LU,\FCOM\}$ as well, $M^{A_M} \equiv M^{A_{CM}}$ holds.
Hence, in the following, we therefore focus only on the  $\CM$-{\bf atomic} case.
% We forbid robot $r$ from performing the $\LK$ operation at time $t_C(s,j)$ because, otherwise,  
% $r$ may observe the new color of $s$ but the old location of $s$,  
% which violates the requirement that the consecutive $\CP$ and $\M$ operations be atomic.

\item $\LCM$-{\bf atomic}-\ASY:
Under this scheduler, each robot performs its consecutive $\LK$, $\Compute$, and $\M$ operations atomically; that is, no robot obtains a snapshot during the interval between the $\LK$, $\Compute$, $\M$ operations of any other robot.  It can be proved that \SSY\ and $\LCM$-{\bf atomic}-\ASY\ are equivalent\cite{FSSW23}.
%Formally, this scheduler satisfies the following condition:  
%$$ \forall r, s \in R,\ \forall i, j \in \mathbb{Z}_{> 0}: \ t_L(r,i) \notin (t_L(s,j), %t_E(s,j)]. $$
\end{itemize}

%In these restricted \ASY\ models, a compounded operation is performed instantaneously; hence, to its  execution  will correspond  a single entry in ${\cal T}$.

%By definition, the \SSY{} scheduler satisfies the conditions of all the above variants of the asynchronous scheduler.

%\begin{observation}

%\end{observation}

%\begin{proof}
%Consider any schedule $(t_L(r,i), t_C(r,i), \allowbreak t_B(r,i), t_E(r,i))_{r \in R,\ i \in %\mathbb{Z}_{> 0}}$ that satisfies the constraints of $\LCM$-atomic-$\ASY{}$.  
%Let $T = \{ t \in \mathbb{R}_{\ge 0} \mid \exists r \in R,\ \exists j \in \mathbb{Z}_{> 0} : %t_L(r,j) = t \}$,  
%and let $T_i$ denote the $i$-th smallest element in $T$.  
%Since the schedule satisfies the conditions of $\LCM$-atomic-$\ASY{}$,  
%the sequence $T_1, T_2, \dots, T_n$ satisfies~\eqref{eq:ssy}.  
%Hence, this schedule also satisfies the constraints of $\SSY{}$.  
%%
%The converse direction follows immediately from the definitions of the two schedulers.
%\end{proof}

% for any $i,j \in \mathbb{Z}_{>0}$ and any $r \in R$, we have 
% $$ 
% ([t_L(r,j),t_E(r,j)]\cap[T_i,T_{i+1}) \neq \emptyset) \to 
% (t_L(r,j) = T_i \land t_E(r,j) < T_{i+1}).
% $$
%$t_L(r,j) = T_i$ yields $t_E(r,j) < T_{i+1}$. 

%\textcolor{blue}{We have to show that $\LCM$-atomic-\ASY\ and \SSY\ are equivalent.}
%Note that, the model where  the $\LK$, $\CP$, and $\M$ operations are considered 
%as a single instantaneous atomic operation
%(thus referable to as  \textbf{$\LCM$-atomic-\ASY})  is 
%obviously equivalent to \SSY. 

 In the following, for simplicity of notation, we shall use the symbols 
 ${F}$, ${R}$, ${S}$, ${A}$, %${A_M}$,
 ${A_{CM}}$, and ${A_{LC}}$ to 
 denote the schedulers \FSY, \RSY,
\SSY, \ASY, %$M$-{\bf atomic}-\ASY,
$\CM$-{\bf atomic}-\ASY, and $\LC$-{\bf atomic}-\ASY, respectively.

Throughout the paper, we assume variable disorientation, chirality and rigidity, unless explicitly stated otherwise.

% Consider now the time complexity. For any algorithm $X$  executed under  semi-synchronous schedulers,   the time complexity is analyzed in terms of  {\em epochs},
%where an epoch  is a minimal
%\footnote{i.e., no sub-sequence has such a property.} 
%sequence of activation rounds   in which all robots are activated and execute $X$ at least once.   
%Clearly, the duration (i.e., the number of rounds) of the epoch might be different.
%\paragraph*{Epochs under ASYNCH (Move-end-based).}
%\textcolor{red}{can be shorted}
We define an \emph{epoch} as one minimal time interval in which every robot
completes at least one \Move (possibly of zero length) since the previous boundary.
Intuitively, each epoch corresponds to the progress of all robots through one
complete \LCM cycle.
The formal definition based on Move-end counters is given Appendix~\ref{app:epoch}.
Using this notion of epoch, we evaluate the efficiency of the simulator by measuring how many simulator epochs are required to reproduce one epoch of the original algorithm.

%\paragraph*{Using epochs for time complexity.}
%The epoch complexity of an execution is the smallest $m$ such that the target predicate
%holds at time $T^{E}_m$. Under weak fairness (each robot is activated infinitely often)
%and no-Zeno, all $T^{E}_m$ are finite.

\subsection{Problems and Computational Relationships}
%-------------------------------------
 Let ${\cal M} = \{\LU, \FC,\FS,\OB \}$  be the set of  models under investigation and 
${\cal S}= \{ F, R, S, A, A_{LC}, %\allowbreak A_{M},
A_{CM}\}$ be the set of  schedulers 
under consideration.

 % In all the above models,
% A {\em configuration} ${\cal C}(\T)$ at time $\T$ is the multiset of the $2$ pairs 
% $(r_i(\T), c_i(\T))$, where $(r_i(\T))$ is the location of robot $r_i$ at time $\T$,
% expressed in a fixed global coordinate system (unknown to the robots), and
% $c_i(\T)$ is the  color of its light at time $\T$. 
% \blue{[Yuichi: I have a concern about this definition of configurations, as it involves the colors of the robots.  
%Later, we define a problem as a temporal geometric predicate that specifies the valid initial, intermediate, and (if any) terminal configurations.  
%However, I believe we all assume that a problem must be defined independently of the robots’ colors---am I right?
%So, I suggest the following paragraph.
% ]}

A \emph{configuration} $\calC (t)$ of $2$ robots at time $t$ is a function $\gamma: \calR_2 \times \mathbb{Z}_{> 0} \to \mathbb{R}^2 \times C$ that specifies the location and color of each robot at time $t$.  
Each location is given with respect to a fixed global coordinate system (which is unknown to the robots).  
Recall that the color set $C$ is a singleton (i.e., $|C| = 1$) in the $\OB$ model; in this case, the colors carry no information.  
%\blue{[Yuichi: I added this sentence because some readers might wonder why $\OB$ robots have colors.  
%Feel free to remove it if you think it's redundant.]}  
A configuration that specifies only the locations (resp. colors) of the robots%, omitting their colors, 
is called a \emph{geometric (resp. color) configuration}; it is defined as a function $\gamma_1: \calR_2 \times \mathbb{Z}_{> 0} \to \mathbb{R}^2$ (resp. $\gamma_2: \calR_2 \times \mathbb{Z}_{> 0} \to C$).
In the following, for robot $r$, we denote $\gamma_1(r,t)$ and $\gamma_2(r,t)$ simply as $r(t)$ and $\mathrm{color}_r(t)$, respectively. 

In the following, to formally define the problem, we introduce predicates that describe the robots’ actions—whether they are stationary or moving—at time $t$;

%\textcolor{blue}{Insert definitions used in the next section. trajectory:r(t), stpo, move etc. For $r \in R$, }
(1) $Stop$ and $Moving$:
\[
\Stop_r(t)\;:\;\exists \varepsilon>0\;\forall \tau\in(t-\varepsilon,t+\varepsilon):\ r(\tau)=r(t).
\quad\]
%\[
%\Move(r,t)\;:\;\neg\,\Stop(r,t).
%\]
\[
\Move_r(t)\;:\;\exists \varepsilon>0\;\exists \tau,\tau'\in(t-\varepsilon,t+\varepsilon)\;:\;r(\tau)\neq r(\tau').
\]

 (2) Moving-time set and number of moves

\[
\MoveSet{r}\;:=\;\{\,t\ge 0\mid \Move_r(t)\,\}.\]
\text{Since }$\MoveSet{r}$\text{ is open, it decomposes into pairwise-disjoint open intervals}
\[
\MoveSet{r}=\bigsqcup_{k=1}^{N_r} I_r^{(k)}
\quad (I_r^{(k)}\text{ open},\ N_r\in\mathbb{N}\cup\{0,\infty\}).
\]
\[
\moves{r}\;:=\;N_r
\quad\text{(assuming finiteness when required)}.
\]

A problem to be solved (or task to be performed) is described by a set of 
{\em temporal geometric predicates} %\blue{[Yuichi: Should be ``predicate'' instead of ``predicates''?]},
 which implicitly define the {\em valid}  initial, intermediate, and (if existing) 
terminal. %\footnote{
A terminal configuration is one in which, once reached, the robots no longer move. %\blue{[Yuichi:
%This sentence is a little confusing, because a configuration does not contain enough information to determine subsequent executions—for example, one robot may be waiting for its next $\CP$ operation, while another may be performing its $\M$ operation toward some point $p$, etc.]}} 
placements, 
as well as restrictions (if any) on the size  $n$ of the set $R$ of robots. 
%We formalize this as follows. 
%Let $\Gamma_n$ denote the set of all placements of $n$ robots.  
%A \emph{trajectory} of $n$ robots is a function $\tau: \mathbb{R}_{\ge 0} \to \Gamma_n$,  
%which specifies the placement of the robots at each time $t$.  
%Even if an initial placement $\gamma \in \Gamma_n$ and an algorithm $\calA$ are fixed,  
%the resulting trajectory is not uniquely determined, as the adversarial scheduler controls the timing of the robots’ LCM actions.  
%Let $T(\gamma, \calA, M, S)$ denote the set of all possible trajectories that may %occur starting from a placement $\gamma$, 
%when all robots execute algorithm $\calA$ under scheduler $S \in \calS$ and model $M \in \calM$.  
%A problem is defined as a pair $P = (\calI, \calT)$,  
%where $\calI \subseteq \bigcup_{n \in \mathbb{N}_{\ge 2}} \Gamma_n$ is the set of possible initial placements,  
%and $\calT$ is the set of legitimate trajectories.\footnote{We introduce $\calI$ for convenience;  
%however, it is essentially unnecessary since $\calI$ can be derived from $\calT$.}  
%An algorithm $\calA$ \emph{solves} the problem $P = (\calI, \calT)$ under scheduler $S \in \calS$ and model $M \in \calM$  
%if $T(\gamma, \calA, M, S) \subseteq \calT$ holds for every $\gamma \in \calI$.
%It is worth noting that a problem concerns only the trajectories of robot locations, % 
%not the colors of the robots.

Given a model $M \in {\cal M}$ and  a scheduler $K\in  {\cal S}$, we denote by
$M(K)$,
 the set of problems solvable 
by robots in $M$ 
under adversarial scheduler $K$.
Let $M_1, M_2\in{\cal M}$ and $K_1, K_2\in{\cal S}$.
\begin{itemize} 
\item We say that  model $M_1$  under scheduler  $K_1$
is {\em computationally not less powerful than} 
model $M_2$ under $K_2$, denoted by
$M_{1}^{K_1} \geq M_{2}^{K_2}$,
 if $M_{1}(K_1) \supseteq M_{2}(K_2)$.

\item We say that 
 $M_1$  under  $K_1$
is {\em computationally more powerful than} 
$M_2$ under $K_2$,  
denoted by
$M_{1}^{K_1} >  M_{2}^{K_2}$,
if 
$M_{1}^{K_1} \geq M_{2}^{K_2}$ 
and
$(M_{1}(K_1) \setminus M_{2}(K_2))  \neq \emptyset$.

\item We say that  $M_1$  under  $K_1$
and  $M_2$  under  $K_2$, are {\em computationally equivalent }, denoted by  
$M_{1}^{K_1} \equiv  M_{2}^{K_2}$,
if $M_{1}^{K_1} \geq M_{2}^{K_2}$ and $M_{2}^{K_2} \geq M_{1}^{K_1}$.

\item Finally, we say that 
$K_1$
$K_2$, are {\em computationally orthogonal} (or {\em incomparable}), denoted by  
$M_{1}^{K_1} \bot  M_{2}^{K_2}$, 
if 
 $(M_{1}(K_1) \setminus M_{2}(K_2))  \neq \emptyset$
 and 
 $(M_{2}(K_2) \setminus M_{1}(K_1))  \neq \emptyset$.

 \end{itemize}
%\section{Previous Results}
%\input{Previous-Results}
\section{Equivalence Between Models}

\subsection{\texorpdfstring{$\LU^R \equiv \LU^A$}{LUALUR}}
%\textcolor{blue}{3-color ref NTW}
This equivalence can be obtained by combining \( \LU^A\equiv\LU^S \)~\cite{DFPSY} with \( \LU^S\equiv\LU^R \)~\cite{FSW19}; however, when restricted to two robots, \cite{NTW24} shows that \( \LU^R \) can be simulated directly in \( \LU^A \), achieving it with three colors in \(6\) epochs. Moreover, this number of colors is optimal~\cite{NTW24}.

\begin{theorem}\label{Sim-LUMIR-by-LUMIA}{\em \cite{DFPSY,FSW19,NTW24}}
%$\forall R \in \mathcal R_2,   \LU_{k}^{RS}(R) \subseteq \FC_{3k^2}^{RS}(R).$
$\LU^{R} \equiv \LU^{A}$.
\end{theorem}

\subsection{\texorpdfstring{$\LU^F \equiv \FC^F \equiv \FS^F$}{LUFFCF} }
% The equivalence between \( \LU^F \) and \( \FC^F \) follows from the fully synchronous setting, in which all robots are activated in every round. In odd rounds each robot copies the other robot’s color, and in the subsequent even rounds it uses that copied color to recover its own color; hence the simulation algorithm can be executed. This simulation uses only two colors and completes in two epochs~\cite{FSW19}. By contrast, the equivalence \( \FC^F \equiv \FS^F \) is obtained only when we restrict to two robots, and it is proved directly without resorting to simulation. Notably, while these two models are separated for \(n\!\ge\!3\), under the two-robot restriction the gap collapses and they become equivalent.

It is shown in \cite{FSW19} the equivalence $\LU^F \equiv \FC^F$ for general $n \ge 2$,
and thus it also holds for our case $n=2$.
The equivalence was established by a simulation using additional colors; that is, if an algorithm solves a problem $P$ with $k$ colors in $\LU^F$,
then there exists another algorithm that solves $P$ with $2k^2$ colors.
In contrast, the equivalence $\LU^F \equiv \FS^F$ does not hold for general $n \ge 2$,
as there exists a separator between $\LU^F$ and $\FS^F${~\cite{apdcm}}.
In this subsection, we show that no such separator exists when restricted to the case $n=2$;
that is, we prove $\LU^F \equiv \FS^F$.
Our simulation from $\FS^F$ to $\LU^F$ requires no additional colors,
and the same holds for the simulation from $\FC^F$ to $\LU^F$.

%The equivalence between \( \LU^F \) and \( \FC^F \) follows from the fully synchronous setting, in which all robots are activated in every round. In odd rounds each robot copies the other robot’s color, and in the subsequent even rounds it uses that copied color to recover its own color; hence the simulation algorithm can be executed. This simulation uses only two colors and completes in two epochs~\cite{FSW19}. By contrast, the equivalence \( \FC^F \equiv \FS^F \) is obtained only when we restrict to two robots, and it is proved directly without resorting to simulation.

% \begin{theorem}\label{Sim-LUMIF-by-FCOMF}{\em \cite{FSW19}}
% %$\forall R \in \mathcal R_2,   \LU_{k}^{RS}(R) \subseteq \FC_{3k^2}^{RS}(R).$
% $\LU^{F} \equiv \FC^{F}$.
% \end{theorem}
%\textcolor{blue}{first part ref FSW and second part not simultion}
%$\forall R \in \mathcal R_2,   \LU_{k}^{RS}(R) \subseteq \FC_{3k^2}^{RS}(R).$
%$\FC^{F} \equiv \FS^{F}$.
\begin{lemma}
\label{lemma:no-cost-sim}
Let $P$ be a problem and $\calA$ an algorithm that solves $P$ with $k$ colors within $T$ epochs in $\LU^F$.
Then there exists an algorithm $\calA_S$ (resp.~$\calA_C$) that solves $P$ with $k$ colors within $T$ epochs under $\FS^F$ (resp.~$\FC^F$).
\end{lemma}

\begin{proof}
% Since $\LU^{F} \ge \FS^{F}$ holds by definition,
% it suffices to show $\LU^{F} \le \FS^{F}$,
% that is, every problem $P$ that can be solved in $\LU^{F}$ 
% can also be solved in $\FS^{F}$.
Suppose that algorithm $\calA$ solves $P$ under $\LU^{F}$.
The adversary may exploit the variable disorientation setting
so that, at any time, each robot observes its partner at point $(0,1)$.
This implies that there exists an algorithm $\calA'$ such that,
at every moment, each robot computes its destination and next color
as if it always observes a snapshot in which its peer is located at point $(0,1)$.
During the execution of $\calA'$, the two robots always have the same color
because of the synchronicity of \FSY.
Therefore, the execution of $\calA'$ under $\LU^{F}$
can be simulated in $\FS^{F}$ and $\FC^{F}$, since
in $\FS^{F}$ (resp.~$\FC^{F}$)
each robot can determine
the color of its peer (resp.~its own color)---it is identical to its own color (resp.~the color of its peer).

\end{proof}

\begin{theorem}\label{equ-FCOMF-by-FSTAF}
$\LU^{F} \equiv \FC^{F} \equiv \FS^{F}$.
\end{theorem}

\subsection{\texorpdfstring{$\LU^R \equiv \FC^R$}{LURFSR}}\label{sec:LUReqFCR}%\textcolor{red}{shorten the description}
%\textcolor{blue}{3-color new }
%\input{LUMIRFCOMR}
%In this section we investigate the relationship between two luminous \textit{restricted-repetition} robots and two \textit{restricted-repetition} robots equipped with external lights. Previous study shows that luminous \textit{restricted-repetition} robots is as powerful as \textit{restricted-repetition} robots equipped with external lights. Another study shows that luminous {\em fully-synchronous} robots is as powerful as {\em fully-synchronous} robots equipped with external lights : luminous {\em fully-synchronous} robots` activation is simulated by repeating step copy (process of copying other robots` lights used in simulated algorithm) and execution (process of execute simulated algorithm with observed other robots` lights). \\

%We now prove that when the number of robots is limited to two, luminous \textit{restricted-repetition} robots is as powerful as \textit{restricted-repetition} robots equipped with external lights with 3 colors. 

The equivalence $\LU^R \equiv \FC^R$ has previously been established via an $\FC^R$-simulation of $\LU^S$~\cite{BFKPSW22} together with an $\LU^S$-simulation of $\LU^R$~\cite{BFKPSW22,NTW24}. 
In this subsection, restricting attention to two robots, we prove the same equivalence by \emph{directly} simulating $\LU^R$ in $\FC^R$, without passing through $\LU^S$.

%The proof is constructive : %we present a $\FC_{3k^2}^{RS}$ protocol SIM that produces a \textit{restricted-repetition} execution of any $\LU_{k}^{RS}$ protocol P.

\paragraph*{Protocol Description}

%The rule of the protocol are presented in Figure \eqref{alg-Sim-LUMI-by-FCOM}, while Figure \ref{SIM-by-3}-(a) shows the transition diagram as the robots change colors.

%The light used by SIM can have three colors: C(opy), M(ove), S(tay). At the beginning, all light are set to C. In each round, activated robot observes other robot`s light used by SIM. If other robot`s light is C, copy other robot`s lights used in simulated algorithm and turn own light to M. If other robot`s light is M, execute simulated algorithm and turn own light to S.
%And if other robot`s light is S, reset own light to C.

% ３状態のアルゴリズム
%\input{sim-lumiR-by-fcomR}

% ３状態のアルゴリズム

%\clearpage

% preamble
% \usepackage{tikz}
% \usetikzlibrary{arrows.meta,positioning,calc}
% \usepackage{subcaption}

\begin{figure}[t]
%\centering

% ---------------- (a) ----------------
%\begin{subfigure}{.38\textwidth}
%\centering
%\begin{tikzpicture}[
%  state/.style={circle,draw,minimum size=14mm,inner sep=0pt},
%  >=Stealth, ->, node distance=30mm
%]
  % nodes
%  \node[state] (cpy) {cpy};
%  \node[state, below left=18mm and 15mm of cpy] (exc) {exc};
%  \node[state, below right=18mm and 15mm of cpy] (rst) {rst};

  % edges with labels (位置調整済み)
%  \draw (exc) -- node[midway,below=2pt] {\small execute $\mathcal{A}$} (rst);
%  \draw (rst) -- node[pos=.55,right=2pt] {\small ack-1} (cpy);
%  \draw (cpy) -- node[pos=.58,left=2pt,align=center] {\small copy peer's\\[-1pt]\small light (echo)} (exc);
%\end{tikzpicture}
%\caption{Local phase automaton (\texttt{exc} $\to$ \texttt{rst} $\to$ %\texttt{cpy} $\to$ \texttt{exc}).}
%\end{subfigure}
%\hfill
% ---------------- (b) ----------------
%\begin{subfigure}{.60\textwidth}
\centering\includegraphics[keepaspectratio, width=10cm]{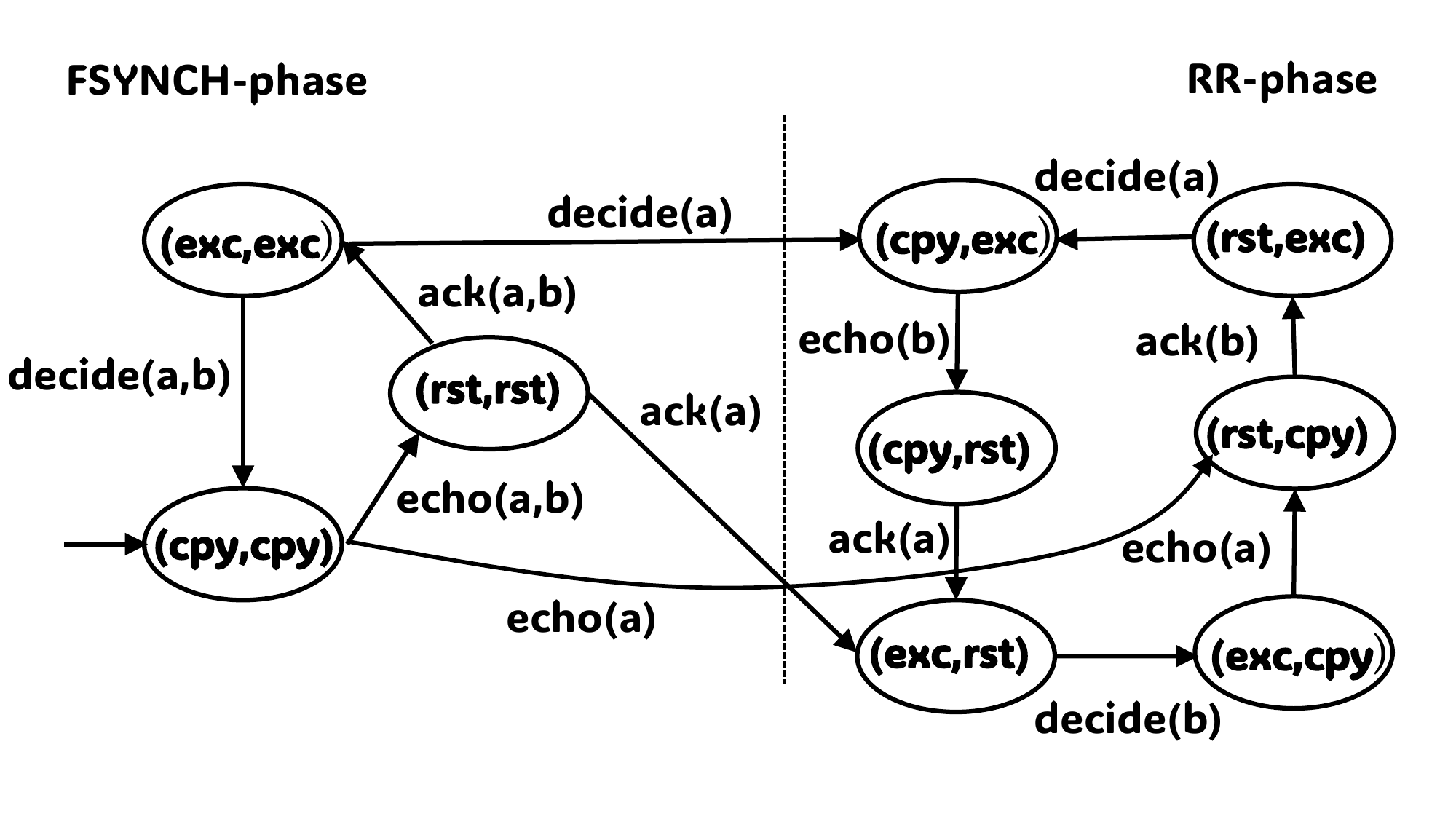}

%\vspace{-1cm}

\caption{Global phase pairs. Left: \FSY; right: \RR.}\label{fig:transition-diagram-SIM}
%\end{subfigure}

%\caption{Phase transition diagrams for \texttt{Sim} with labels \texttt{exc}/\texttt{cpy}/\texttt{rst}.}
\end{figure}

%\clearpage

%\paragraph*{Correctness}
%\paragraph*{High-level overview.}
We simulate any two-robot algorithm $\mathcal{A}$ written for
$\LU^R$ inside $\FC^R$
using three lights per robot $r$ and the other robot is denoted by $peer$:
\begin{itemize}
  \item \texttt{r.my.light} — the light value that $\mathcal{A}$ exposes for $r$,
  \item \texttt{r.your.light} — the last observed light of the peer (used as input to $\mathcal{A}$),
  \item \texttt{r.phase} $\in \{\textsf{exc},\textsf{cpy},\textsf{rst}\}$\footnote{The name \textsf{exc},\textsf{rst}, and\textsf{cpy} mean $execute$, $reset$, and $copy$.} — a short handshake for turn delimiting, initially set to {\textsf{cpy}}.
\end{itemize}
Under $\mathrm{RSYNCH}$, activations strictly alternate (optionally simultaneous first activations).
On each activation, the active robot executes:
\begin{itemize}
  \item \textbf{If} \texttt{peer.phase = exc} \textbf{(decision turn):}
        run $\mathcal{A}$ on the snapshot
        \[
          (\text{positions},\ L_{\text{self}},\ L_{\text{peer}})
          \quad\text{with}\quad
          L_{\text{self}} := \texttt{peer.your.light},\quad
          L_{\text{peer}} := \texttt{peer.my.light},
        \]
        set the returned \texttt{r.my.light} and destination \texttt{r.des};
        then set \texttt{r.phase} $\leftarrow$ \textsf{cpy}.
          \item \textbf{If} \texttt{peer.phase = cpy} \textbf{(peer just decided):}
        set \texttt{r.your.light} $\leftarrow$ \texttt{peer.my.light} \ (\emph{echo})\ and
        \texttt{r.phase} $\leftarrow$ \textsf{rst}.
  \item \textbf{If} \texttt{peer.phase = rst} \textbf{(peer just copied):}
        set \texttt{r.phase} $\leftarrow$ \textsf{exc} (acknowledge).
\end{itemize}

In \emph{Move}, the robot moves rigidly to \texttt{r.des} (defaulted to \texttt{r.pos} at the start of \emph{Compute}).

The transition diagram of state transitions is presented in Fig.~\ref{fig:transition-diagram-SIM}, where when changing from \FSY\ to \RR\, we may assume that the robot $a$ is the one that becomes activated. Note that since the initial value of $r.phase$ is \textsf{cpy}, the robot $r$ that is activated first copies $peer.my.light$ into $r.your.light$, thereby enabling the execution of $\LU$ in the next decision turn.

When the execution starts in the \FSY\ mode (the left side of Fig.~\ref{fig:transition-diagram-SIM}), the robot $r$ copies  \texttt{peer.my.light} to \texttt{r.your.light}, move to the executable state, and repeatedly perform $\mathcal{A}$.
Since each robot’s \texttt{r.your.light} is correctly updated, the behavior of $\LU$ can be faithfully simulated.

When switching to the \RR\ phase, regardless of the state from which the transition occurs, the correct light values have already been updated, and from that point, through alternating activations, each robot can correctly execute the $\LU$ algorithm (the right side of Fig.~\ref{fig:transition-diagram-SIM}).

%When turns alternate, the peer never looks or moves during our turn. Thus, the peer's next \emph{Look}
%happens after we have published our new light and finished our move - exactly the visibility assumed by $\mathrm{\LU}^{\mathrm{R}}$. 
%\emph{Notation:} keep the destination name consistent (use either \texttt{r.des} everywhere
%or \texttt{my.des} everywhere).\paragraph*{Correctness sketch.}

The simulator maintains three public lights per robot to coordinate the phases
of $\mathcal{A}$’s atomic steps under \RSY\ alternation.
Each robot’s phase variable cycles as
$ \textsf{exc}\!\to\!\textsf{rst}\!\to\!\textsf{cpy}\!\to\!\textsf{exc}$,
so that exactly one robot performs a \emph{decision turn} at a time.
During this turn, the robot runs $\mathcal{A}$, publishes its new light,
and moves rigidly while the peer is inactive.
Hence, every decision turn corresponds to one atomic step of $\mathcal{A}$ in $\LU^R$.
The following two facts make the correctness immediate from the transition diagram
(Fig.~\ref{fig:transition-diagram-SIM}):

\begin{itemize}
  \item \textbf{Atomic visibility.}  A robot’s light update is always observed by the peer
        only \emph{after} the move is completed, exactly as assumed in $\LU^R$.
  \item \textbf{Phase discipline.}  Each cycle of three phase changes
        ($\exc\!\to\!\rst\!\to\!\cpy\!\to\!\exc$) realigns the two robots and
        guarantees that a new decision turn occurs within three activation.
\end{itemize}

Therefore, the simulator faithfully reproduces the stepwise behavior of
$\mathcal{A}$: one epoch of $\mathcal{A}$ is realized by $3$ epochs in the simulator.%\footnote{In fact, the occurrence of $5$ epochs occurs only during the transition from $(\rst,\rst)$ to $(\cpy,\rst)$, that is, when changing from \FSY\ to \RR.
%Within the \FSY\ phase and within the \RR\ phase themselves, only three epochs are involved.}
The proof of the invariants and the detailed case analysis appear in Appendix~\ref{app:Sim-LU-by-FC-3cols}.

\begin{lemma}\label{le:Sim-LU-by-FC-3cols}
Any algorithm $\mathcal{A}$ with $k$ colors and within $T$ epochs in $\LU^R$ can be
simulated by a simulator with $3k^2$ colors and within $3T$ epochs in $\FC^R$.
%The protocol SIM is correct, i.e. any execution of protocol SIM in $\FC_{3k^2}^{RS}$ corresponds to a possible execution of $\LU_{k}^{RS}$.
\end{lemma}

\begin{theorem}\label{th:Sim-LUMIR-by-FCOMR}
%$\forall R \in \mathcal R_2,   \LU_{k}^{RS}(R) \subseteq \FC_{3k^2}^{RS}(R).$
$\LU^{R} \equiv \FC^{R}$
%\end{corollary}
\end{theorem}
\subsection{\texorpdfstring{$\FC^{S} \equiv \FC^{A_{LC}}$}{FCSFCALC} and \texorpdfstring{$\FC^{A_{CM}} \equiv \FC^A$}{FCAMFCA}}

These two equivalences can be achieved by a single simulation algorithm~\cite{FSSW23}.
Here, by exploiting the fact that the system consists of only two robots, we refine the construction into a more efficient algorithm that requires fewer colors and epochs.
\input{sim-abst}

\begin{lemma}\label{le:Sim-FCAM-by-FCA-12-k-cols}
Any algorithm $\mathcal{A}$ with $k$ colors and within $T$ epochs in $\FC^{A_{CM}}$ can be
simulated by a simulator with $7k$ colors and within $4T$ epochs in $\FC^A$.
%The protocol SIM is correct, i.e. any execution of protocol SIM in $\FC_{3k^2}^{RS}$ corresponds to a possible execution of $\LU_{k}^{RS}$.
\end{lemma}
%\begin{proof}12-colors --> 7 colors. 
%\end{proof}

\begin{theorem}\label{th:Sim-FCAM-by-FCA-12-k-cols}
%$\forall R \in \mathcal R_2,   \LU_{k}^{RS}(R) \subseteq \FC_{3k^2}^{RS}(R).$
$\FC^{A_{CM}} \equiv \FC^{A}$.
\end{theorem}
%\textcolor{blue}{7-color new but this time refer to OPODIS2023}
%\input{old/Sim-FCOMA-by-FCOMAM}
To prove $\FC^{S} \equiv \FC^{A_{LC}}$,
we can use the same simulation algorithm.
%to show the equivalence between $A_{LC}$ and \SSY\ in $\FC$.
%
%\color{red} the simulation algorithm is the same. 
%\color{blue}
%
Since $\FC^{A_{LC}}\leq   \FC^{S}$  
by definition, to prove $\FC^{A_{LC}}\equiv\FC^{S}$, we need to show that
every problem solvable by a set of $\FC$ robots  under \SSY\ can also be solved 
 under $A_{LC}$.

The {\em simulation} algorithm for $\FC$ robots that enables them to correctly execute under $A_{LC}$ any protocol designed for \SSY\ is precisely {\tt SIM}($\mathcal{A}$).
This algorithm executes under \ASY\ any ${\cal A}$ designed for $A_{CM}$.
Indeed, any asynchronous execution of {\tt SIM}($\mathcal{A}$) corresponds to an execution of ${\cal A}$ under $A_{CM}$.
If the scheduler is restricted to $A_{\alpha}$, then each such execution yields one of ${\cal A}$ under $A_{CM}$ and $A_{\alpha}$.
Hence, executing {\tt SIM}($\mathcal{A}$) under $A_{LC}$ produces an execution of ${\cal A}$ under $LC$- and $CM$-atomic \ASY, i.e., under $LCM$-atomic \ASY, which is equivalent to \SSY, proving Theorem~\ref{th:FCALC=FCS}.
 
%The {\em simulation} algorithm 
%for $\FC$ robots  that allows them to correctly execute under $A_{LC}$ 
%any  protocol designed to work under \SSY,
%is actually precisely  the simulation
%algorithm {\tt SIM}($\mathcal{A}$),
%sim-$\FC^{A_{CM}}$-by-$\FC^{A}$(${\cal A}$)}, \color{red} this is a terrible name; since this is the only simulation protocol used at all, it can be called something simple, like SIM 
% that 
%executes under \ASY\ any algorithm ${\cal A}$ designed to work under $A_{CM}$.
%
%To understand why this is the case, observe that, as we have shown, {\em any} asynchronous execution of {\tt SIM}($\mathcal{A}$)
%produces a specific execution of ${\cal A}$ under $A_{CM}$. If the executions of {\tt SIM}($\mathcal{A}$) were not arbitrary (i.e., under \ASY) but  under a restricted asynchronous scheduler (say $A_{\alpha}$), then each such execution would clearly produce a specific execution of ${\cal A}$ under the asynchronous scheduler which satisfies both $CM$ and $\alpha$.
%
%Thus, the execution of {\tt SIM}($\mathcal{A}$) under $A_{LC}$,  will produce an execution of ${\cal A}$
%that satisfies both $LC$-{\bf atomic} and $CM$-{\bf atomic}; that is, an  execution under {$LCM$}-{\bf atomic}-\ASY\ (which is equivalent to \SSY),
%giving Theorem~\ref{th:FCALC=FCS}.

\begin{theorem}  
\label{th:FCALC=FCS}
%$\forall R \in \mathcal R_2,   \LU_{k}^{RS}(R) \subseteq \FC_{3k^2}^{RS}(R).$
$\FC^{S} \equiv \FC^{A_{LC}}$
\end{theorem}
%\subsection{\texorpdfstring{$\LU^F \equiv \FS^F$}{LUFFSF}}

\section{Separation between Schedulers}
\subsection{Separation \RSY\ from \FSY}
The separation between \FSY\ and \RSY\ is established by the following CGE*.%\textcolor{red}{CGE* only}

\vspace{-0.3cm}

\paragraph*{\textit{CGE*} (Perpetual Center of Gravity Expansion)~\cite{FSW19}}
Let $R$ be a set of robots with $|R|\ge2$.  
Let $P=\{(x_1,y_1), (x_2,y_2), \ldots, (x_n,y_n)\}$ denote their initial positions, and let
$c=(c_x,c_y)$ be the coordinates of the \emph{center of gravity} (CoG) of $P$ at time $t=0$.  

In the \textit{CGE*} problem, each robot $r_i\in R$ repeatedly executes the following operation forever:
it moves from its current position $(x_i,y_i)$ directly to a new position
%\[
$(f(x_i,c_x),\, f(y_i,c_y)) \quad\text{where}\quad f(a,b)=\lfloor 2a-b \rfloor.$
%\]
In other words, each robot moves straight away from the initial center of gravity, 
doubling its initial distance from it.  
Once every robot reaches its target, the configuration becomes stationary, and the process restarts from the new positions, expanding indefinitely.

%\paragraph*{\textit{CGE} (Center of Gravity Expansion)\cite{BFKPSW22}}
%\paragraph*{\textit{CGE*} (Perpetual Center of Gravity Expansion)~\cite{FSW19}}
%    Let $R$ be a set of robots with $|R| \geq 2$. Let $P = \{(x_1, y_1), (x_2, y_2), \ldots, (x_n, y_n)\}$ be the set of their positions and let $c = (c_x, c_y)$ be the coordinates of the \textit{CoG} (Center of Gravity) of $P$ at time $t = 0$. The \textit{CGE*} problem repeat the following process forever that each robot $r_i \in R$ to move from its initial position $(x_i, y_i)$ directly to $(f(x_i, c_x), f(y_i, c_y))$ where $f(a, b) = \lfloor 2a - b \rfloor$, away from the Center of Gravity of the initial configuration so that each robot doubles its initial distance from it and no longer moves.

% This is the same as CGE, where however after each expansion, the robots have to repeat the same process from the new configuration.

Since $\mathrm{CGE}^*$ already holds for two robots, the separations below are valid for $n=2$.

%For every $M\in\{\LU,\FC,\FS,\OB\}$,
%\[
%\mathrm{CGE}^* \in M^F
%\quad\text{and}\quad
%\mathrm{CGE}^* \notin M^R.
%\qquad\text{(for $M=\LU$, see~\cite{FSW19})}.
%\]
%In addition, for $n=2$ the problem $\mathrm{CGE}$ separates
%$\LU^F\ (\equiv \FC^F \equiv \FS^F)$ from $\OB^F$.

\begin{theorem}[Separation via CGE$^*$; even for $n=2$]~{\em \cite{BFKPSW22}}

For every $M\in\{\LU,\FC,\FS\}$, %the following holds:
%\begin{enumerate}
  %\item[\textnormal{(CGE$^*$)}] $\mathrm{CGE}^* \in M^{F}$ and $\mathrm{CGE}^* \notin M^{R}$.
    $\mathrm{CGE}^* \in M^{F}$ and $\mathrm{CGE}^* \notin M^{R}$.
%  \item[\textnormal{(CGE)}] $\mathrm{CGE} \in \LU^{F}\ (\equiv \FC^{F} \equiv \FS^{F})$ and $\mathrm{CGE} \notin \OB^{F}$.
%  \item[(2)] $\mathrm{CGE} \in \LU^{F}\ (\equiv \FC^{F} \equiv \FS^{F})$ and $\mathrm{CGE} \notin \OB^{F}$.
%\end{enumerate}
%Moreover, $\mathrm{CGE}^* \in \OB^{F}$.
\end{theorem}

%\begin{table}[t]
%\centering
%\newcommand{\OK}{\checkmark}
%\newcommand{\NO}{\(\times\)}
%\caption{Membership/separation summary (holds already for $n=2$).}
%\label{tab:cge-summary}
%\begin{tabular}{lccccccc}
%\toprule
% & $\LU^F$ & $\LU^R$ & $\FC^F$ & $\FC^R$ & $\FS^F$ & $\FS^R$ & $\OB^F$ \\
%\midrule
%CGE$^*$ & \OK & \NO & \OK & \NO & \OK & \NO & \OK \\
%CGE     & \OK\footnotemark & –   & \OK\footnotemark & –   & \OK\footnotemark & –   & \NO \\
%\bottomrule
%\end{tabular}
%\end{table}
%\footnotetext{By $\LU^F \equiv \FC^F \equiv \FS^F$.}

%\subsubsection{\texorpdfstring{$\OB$}{OBLOT}(RDV)}
%\subsubsection{\texorpdfstring{$\LU$ $\FC$ $\FS$ and $\OB$}{LUFCFS}(CGE) or CGE*}
\subsection{Separation \SSY\ from \RSY}
%\textcolor{red}{We don't need this section!}
The separation between \RSY\ and \SSY\ is established by the following SRO.

\vspace{-0.3cm}

\paragraph*{\textit{SRO} (Shrinking ROtation)~\cite{FSW19}}
The \textit{SRO} problem involves two robots that repeatedly transform their relative
configuration by clockwise rotations and occasional size reductions.
At each step, the segment connecting the two robots either rotates by $90^\circ$
around its midpoint (pure rotation) or rotates by $45^\circ$ about one endpoint
while its length is scaled by $1/\sqrt{2}$ (shrink rotation).
These operations are performed so that both robots always remain within
the previous square frame defined by their former positions.
The process generates a perpetual, nested sequence of rotated and shrinking
configurations, producing a clockwise “turn-and-shrink’’ motion.
(The full formal definition is provided in Appendix~\ref{app:sro}.)

The essence of \textit{SRO} is that, when the two robots operate either simultaneously or alternately (\RSY), correct shrinking can be achieved even in the $\OB$ model.
In contrast, under the \SSY model, consecutive activations of a single robot cause violations of the \textit{SRO} conditions if the robot relies only on its own memory (as in $\FS$) or only on communication (as in $\FC$)~\cite{FSW19}.
Since in the $\LU$ model the \RSY\ and \ASY\ variants are equivalent, the \textit{SRO} problem is solvable in $\LU^{A}$.

%Intutitive proof
\begin{theorem}[Separation via SRO]~{\em \cite{FSW19}}

For every $M\in\{\FC,\FS,\OB\}$, %the following hold:
%\begin{enumerate}
  %\item[\textnormal{(CGE$^*$)}] $\mathrm{CGE}^* \in M^{F}$ and $\mathrm{CGE}^* \notin M^{R}$.
    %\item[(1)] 
    $\mathrm{SRO} \in M^{R}$ and 
%  \item[\textnormal{(CGE)}] $\mathrm{CGE} \in \LU^{F}\ (\equiv \FC^{F} \equiv \FS^{F})$ and $\mathrm{CGE} \notin \OB^{F}$.
%  \item[(2)] 
$\mathrm{SRO} \notin M^{S}$.
%\end{enumerate}
%Moreover, $\mathrm{CGE}^* \in \OB^{F}$.
\end{theorem}

\subsection{Separation \texorpdfstring{$CM$-atomic-\ASY}{MatomicASY} from \SSY}
The separation between \SSY\ and $CM$-{\bf atomic} \ASY\ is established by the following \textit{MCv}.

%\begin{definition}
\begin{definition}[\textbf{MONOTONE CONVERGENCE} ($\MCv$)]
\label{def:MCv}
Two robots, $a$ and $b$,  are initially positioned at distinct locations; they  must converge to a common location while never increasing the distance between them. In other words, an algorithm solves $\MCv$ if and only if it satisfies the following predicate:
\begin{align*}
MCV \equiv\; &\Bigl[ \left(\exists \ell \in \mathbb{R}^2,\; \forall \epsilon > 0,\; \exists T \ge 0,\; \forall t \ge T:\; |a(t) - \ell| + |b(t) - \ell| \le \epsilon\right) \\
&\quad \quad \quad \quad \quad \land\; \left(\forall t, t' \ge 0:\; t \le t' \to |a(t)-b(t)| \ge |a(t')-b(t')|\right) \Bigr].
\end{align*}
\end{definition}

%{\bf MONOTONE LINE CONVERGENCE} $($\MLCv$)$
% The two robots, $r$ and $q$ must  solve the {\tt Collisionless Line Convergence} problem without 
% ever increasing the distance between them.%;
 %that is, at any time interval, 
 %the distance must not increase.
%\end{definition}
%\noindent In other words, an algorithm solves
%\MLCv\   iff it
%satisfies the following predicate:
%$$MLC \equiv [ CLC\  {\bf and}\ \{\forall t'\geq t, |r(t')-q(t')| \le |r(t)-q(t)|\}]$$
%\begin{align}
%\begin{split}
%CLC \equiv \Big[
%&\{\exists \ell \in \mathbb{R}^2, \forall\epsilon\geq 0, \exists T\geq 0, \forall t \ge T:  |r(t)-\ell|+ |q(t)-\ell| \le \epsilon\}, \\
%&{\textbf{and}}\  \{\forall t\geq 0: r(t),q(t)\in\overline{r(0)q(0)}\}, \\
%&{\textbf{and}}\ \{\forall t\geq 0: %\\
%&\quad 
%dis(r(0),r(t))\leq dis(r(0),q(t)), %\\
%&\quad 
%dis(q(0),q(t))\leq dis(q(0),r(t))\} \Big]
%\end{split}
%\nonumber
%\end{align}

\begin{theorem}[Separation via \MCv]~{\em \cite{FSSW23}}

For every $M\in\{\FC,\FS,\OB\}$, %the followings hold:
%\begin{enumerate}
  %\item[\textnormal{(CGE$^*$)}] $\mathrm{CGE}^* \in M^{F}$ and $\mathrm{CGE}^* \notin M^{R}$.
    %\item[(1)] 
    $\mathrm{MCv} \in M^{S}$ and 
%  \item[\textnormal{(CGE)}] $\mathrm{CGE} \in \LU^{F}\ (\equiv \FC^{F} \equiv \FS^{F})$ and $\mathrm{CGE} \notin \OB^{F}$.
%  \item[(2)] 
$\mathrm{MCv} \notin M^{A_{CM}}$.
%\end{enumerate}
%Moreover, $\mathrm{CGE}^* \in \OB^{F}$.
\end{theorem}
Note that since $\LU^{S}\equiv \LU^A$, \MCv\ can be solved in $\LU^A$.

\subsection{Separation \ASY\ from \texorpdfstring{$CM$-atomic-\ASY}{ASYMaotmicASY}}
% Dyadic stopped-time distance as a temporal geometric predicate / problem
% ==== Dyadic-multiple variant (scale-invariant) ====
%In this subsection, we show a $2$-robot problem that separates $M^{A_M}$ from $M^{A}$ for every $M\in\{\FS,\OB\}$. 
Here, for $M\in\{\FS,\OB\}$, we present a two-robot problem that separates $M^{A_{CM}}$ from $M^{A}$. Previously, this separation was achieved using Trapezoid Formation (TF)~\cite{FSSW23}, but that construction requires four robots.
We prove this separation using the following new problem.

%, the followings hold:
%\begin{enumerate}
  %\item[\textnormal{(CGE$^*$)}] $\mathrm{CGE}^* \in M^{F}$ and $\mathrm{CGE}^* \notin M^{R}$.
%    \item[(1)] $\mathrm{MLCv} \in M^{S}$.
%  \item[\textnormal{(CGE)}] $\mathrm{CGE} \in \LU^{F}\ (\equiv \FC^{F} \equiv \FS^{F})$ and $\mathrm{CGE} \notin \OB^{F}$.
%  \item[(2)] $\mathrm{MLCv} \notin M^{A_{M}}$.
%\end{enumerate}
\begin{definition}[\textbf{Dyadic-Multiple Stopped-Distance} ($DMSD$)]
Two robots $a$ and $b$ start at positions $a(0),b(0)\in\mathbb{R}^2$ with initial distance
$
D_0 \;:=\; |a(0)-b(0)| >0\,.
$
%Unless stated otherwise, we assume $D_0>0$.  
The \DMSD\ problem asks for an algorithm that guarantees that \emph{at every time $t$ when both
robots are simultaneously stopped}, the distance between them is a dyadic rational multiple of
$D_0$. The formal description is as follows;

The set of \emph{dyadic rationals} is
$
  \Dyad \;:=\; \left\{ \frac{m}{2^k} \;\middle|\; m\in\mathbb{Z},\ k\in\mathbb{N} \right\}.
$
%\end{definition}
%\paragraph*{Positions and no-Zeno.}
%Fix a two-robot set $\{a,b\}$.
%For each robot $r\in\{a,b\}$, let $r(t)\in\mathbb{R}^2$ denote its position at time $t\ge 0$.
%We assume \emph{no Zeno behavior}: on every compact time interval, each $r$ has only finitely
%many switches between moving and being stopped.
%\begin{definition}[Stopped-at-$t$ and joint stops]
%For $r\in\{a,b\}$ and $t\ge 0$, say that $r$ is \emph{stopped at $t$} if
%\[
%  \Stop_r(t)\ :\iff\ \exists\,\varepsilon>0\ \ \forall s\in[t,t+\varepsilon):\ r(s)=r(t).
%\]
The set of \emph{joint stops} of $\{a,b\}$ is
$
  \Sigma_{ab}\ :=\ \{\, t\ge 0 \mid \Stop_a(t)\ \wedge\ \Stop_b(t)\, \}.
$
An algorithm \emph{solves} \DMSD\ if every execution it generates satisfies the predicate
$DMSD$, where
$DMSD$ holds iff
$
  \forall\, t\in\Sigma_{ab}:\quad \frac{|a(t)-b(t)|}{|a(0)-b(0)|} \in\ \Dyad 
$
($|a(0)-b(0)|>0$).
\end{definition}

\paragraph{Possibility in $CM$-atomic \ASY.}%\textcolor{blue}{algorithm:goto the midpoint}
We show that DMSD is solvable in $\OB$ under $CM$-{\bf atomic} \ASY\ that all%the assumption that all
moves are non-observable by the algorithm {\bf Go-to-midpoint}. Nevertheless,  a robot may initiate its move based on an \emph{old} snapshot, while the other robot may have already changed its position in the meantime.

\begin{lemma}[Distance update]\label{lem:pos-DMSD}
Let $D$ be the distance at some stopped configuration that a robot $a$ \emph{looked at} when it computed its midpoint target. Suppose before $a$ actually executes that Move, the other robot $b$ performs $r\ge 0$ fresh midpoint Moves (each halves the distance it sees). Then the new distance $D'$ after $a$ finally moves is
\[
  D'=\psi(r)\,D,
  \qquad
  \psi(0)=\frac{1}{2},\qquad
  \psi(r)=\frac{2^{r-1}-1}{2^{r}}\;\;(r\ge 1).
\]
In particular, $\psi(1)=0$, and for $r\ge 2$ the numerator $2^{r-1}-1$ is odd.
\end{lemma}

\begin{proof}[Sketch]
Reduce to one dimension with points $p_a<p_b$. If $r=0$ the mover uses the current snapshot, so $D'=D/2$. If $r\ge 1$, then after the $r$ fresh Moves of $b$ we have $b\mapsto p_a+D/2^{r}$; the mover’s (old) midpoint target is $(p_a+p_b)/2$, hence
\[
  D'=\Bigl|\frac{p_a+p_b}{2}-\Bigl(p_a+\frac{D}{2^r}\Bigr)\Bigr|
     = \Bigl|\frac12-\frac1{2^r}\Bigr|\,D
     = \frac{2^{r-1}-1}{2^{r}}\,D.
\]
\end{proof}

\noindent Iterating the lemma shows that every stopped-time distance has the form
\[
  D_t
  \;=\;
  D_0 \prod_{\ell=1}^{t} \psi(r_\ell)
  \;\in\;
  \mathbb{Z}\!\left[\tfrac12\right],
\]
i.e., a dyadic rational $m/2^{k}$. Moreover, unless some factor $\psi(1)=0$ occurs (which yields distance $0$), the overall numerator is odd. Thus, we have the following lemma.

\begin{lemma}
DMSD is solvable in $\OB^{A_{CM}}$.
\end{lemma}

%\paragraph{General ASYNCH (Look allowed during a Move): a non-dyadic example.}
%Now Moves are not atomic and a robot may \emph{look while the other is in transit}. Start with $a_0=0$, $b_0=1$. Robot $a$ looks, plans to move to $1/2$, and starts moving. While $a$ is en route, at some intermediate position $\lambda\in(0,1/2)$, robot $b$ looks and moves to the instantaneous midpoint
%\[
%  m_b=\frac{1+\lambda}{2}.
%\]
%Then $a$ finishes its Move to $1/2$. The stopped-time distance is
%\[
%  \bigl|\,\tfrac12 - \tfrac{1+\lambda}{2}\,\bigr| \;=\; \frac{\lambda}{2}.
%\]
%Choosing, e.g., $\lambda=\tfrac{1}{3}$ yields distance $1/6$, whose denominator is not a power of two; hence the dyadic property fails when Looks may occur during Moves.

%\newcommand{\Dyad}{\mathbb{D}}
%\newcommand{\R}{\mathbb{R}}

%\paragraph{Dyadic multiples.}
%Let the set of dyadic rationals be
%\[
%  \Dyad \;:=\; \left\{ \frac{m}{2^k} \;\middle|\; m\in\mathbb{Z},\ k\in\mathbb{N} \right\}.
%\]
%For initial positions $A_0,B_0\in\R^2$ with $D_0:=\lVert B_0-A_0\rVert>0$, the \textsc{DMSD} requirement is:
%\emph{at every time $t$ when both robots are simultaneously stopped, the distance $\lVert a(t)-b(t)\rVert$ is in $\Dyad\cdot D_0$.}

\begin{lemma}[Impossibility of DMSD in \textsc{\ASY}]\label{le:dmsd-rigid-asynch}
%Consider the standard asynchronous Look--Compute--Move model with \textsc{RIGID} moves (each commanded move reaches its destination), anonymous robots, and either \textsc{OBLOT} (no internal memory) or \textsc{FSTA} (finite internal memory not externally visible). 
%For any deterministic algorithm on $\OB$ or $\FS$ that commands a positive move on some instance, there exist an initial configuration and a fair adversarial schedule such that the \textsc{DMSD} property is violated. 
No nontrivial deterministic algorithm solves \textsc{DMSD} in $\FS^A$ or $\OB^A$.
\end{lemma}
\begin{proof}[Proof sketch]
If a robot is observed while it is moving, the adversary can choose the observation timing so that the resulting simultaneous stop occurs at a distance that is not a dyadic multiple of the initial one. Hence, any deterministic algorithm that commands a positive move inevitably violates the \textsc{DMSD} condition under some fair asynchronous schedule. The detailed construction is given in Appendix~\ref{app:imp-DMSD}.
\end{proof}

%\paragraph{Instantiation (midpoint rule).}
%For the common midpoint rule $\lambda=\tfrac12$, the formula specializes to
%\[
%  \frac{\lVert A'-B'\rVert}{D_0} \;=\; \frac{x}{4},
%\]
%and choosing any $x\in(0,1)$ with $x/4\notin\Dyad$ yields the violation.

\begin{theorem}[Separation via \DMSD]

For every $M\in\{\FS,\OB\}$, %the followings hold:
%\begin{enumerate}
  %\item[\textnormal{(CGE$^*$)}] $\mathrm{CGE}^* \in M^{F}$ and $\mathrm{CGE}^* \notin M^{R}$.
    %\item[(1)] 
    $\DMSD \in M^{A_{CM}}$ and
%  \item[\textnormal{(CGE)}] $\mathrm{CGE} \in \LU^{F}\ (\equiv \FC^{F} \equiv \FS^{F})$ and $\mathrm{CGE} \notin \OB^{F}$.
%  \item[(2)]  
$\DMSD \notin M^{A}$.
%\end{enumerate}
%Moreover, $\mathrm{CGE}^* \in \OB^{F}$.
\end{theorem}
Note that since $\FC^{A_{CM}}\equiv \FC^A$, \DMSD\ can be solved in $\FC^A$ and $\LU^A$.

%\section{Landscape between \texorpdfstring{$\FC$ and $\FS$}{FCandFS}}

%\subsection{Separation \texorpdfstring{\RSY\ from \SSY}{RSYSSY}(SRO)}
\section{Orthogonality between \texorpdfstring{$\FC$ and $\FS$}{FCandFS}}
In this section, we show that the twelve pairs $\bigl(\FC^{\alpha},\FS^{\beta}\bigr)$ with
$\alpha\in\{S,A_{CM},A\}$ and $\beta\in\{R,S,A_{CM},A\}$ are incomparable.
It suffices to exhibit two problems $P_1$ and $P_2$ such that
\[
P_1 \notin \FC^{S} \ \text{while}\ P_1 \in \FS^{A},
\qquad
P_2 \notin \FS^{R} \ \text{while}\ P_2 \in \FC^{A}.
\]

\subsection{Single Move}
$P_1$ is the \emph{Single-Move} (SM) problem~\cite{FSSW23}: two robots must each move exactly once to a different location.
(Its formal definition is stated in Appendix~\ref{app:SM}. It has been shown that SM exhibits the required properties of $P_1$~\cite{FSSW23}.

%For this problem, the following properties hold.

\begin{theorem}~{\em \cite{FSSW23}}%[Proof in {\bf Appendix}~\ref{ap:SMLS}]
\label{lem:SMLS}\
%\begin{enumerate}
    %\item  
    $\text{SM} \not \in \FC^{S}$ and
 %   under rigid-movement and agreement of chirality.
    %\item  
    $\text{SM}  \in \FS^{A}$.
   % under rigid-movement.
%\end{enumerate}
\end{theorem}

\subsection{Rendezvous with One Step or Asymmetric Movement}

The problem corresponding to $P_2$
in the case of three or more robots is represented by CYC~\cite{BFKPSW22}.
In such cases, the colored-configuration formed by 
$n$ robots cannot be recognized by $\FS$, and hence cannot be solved under \RSY\ (or even under \FSY).
In contrast, this configuration can be easily recognized in $\FC$, and therefore the problem is solvable under \ASY.

%However, when there are only two robots, such a method cannot be employed .
However, when only two highly symmetric robots exist, the conventional methods cannot be applied.%when only two robots exist, such a method cannot be employed;
In the $\FS$ model, symmetry breaking is impossible,
whereas in the $\FC$ model it becomes possible.
Therefore, we must identify a problem that can distinguish these two cases.

Here, we combine two kinds of problems—one that separates $\FC^{R}$
 from $\FS^{R}$, and another that remains unsolvable as long as the robots move simultaneously (even under $\LU^F$).
We show that this combined problem can serve as $P_2$.

The former problem corresponds to the rendezvous-in-one-move problem (
$RDV_1$),
while the latter requires that one robot stops after at most one move and the other moves at least twice (\textit{Asymmetric Move:AM}).
Formally, these problems are defined as follows.
%\color{black}\subsection{Rendezvous with One Step or Different Movement}
% --- Base notation on trajectories ---
%\newcommand{\R}{\mathbb{R}}

% Robot i \in {1,2} has a trajectory X_i : R_{\ge 0} -> R^2.
% Local "moving" / "stopped" predicates at time t:
%\newcommand{\Stop}[2]{\mathrm{Stop}(#1,#2)}
%\newcommand{\Move}[2]{\mathrm{Move}(#1,#2)}

% (1) Local stopped / moving at time t
% Stop(i,t) : i is locally constant around t

% (3) Eventually stopped
%\[
%\Stopped{r}\;:\;\exists T\ge 0\ \forall t\ge T:\ r(t)=r(T).
%\]

% (4) Rendezvous
%\[
%\Rend\;:\;\exists t\ge 0:\ a(t)=b(t) \land\ \Stopped{a} \land\ \Stopped{b}.
%\Rend:\ \exists T\ge0\ \forall t\ge T:\ a(t)=b(t).
%\]

% (5) Single-move predicate for one robot i
%   -- exactly one moving interval, stopped elsewhere, and net displacement nonzero.
%\newcommand{\SingleMove}{\mathrm{SingleMove}}
%\[
%\SingleMove(r)\;:\;\exists s<e\ \text{s.t.}\ 
%\begin{cases}
%\forall t<s:\ \Stop(r,t),\\
%\forall t\in(s,e):\ \Move(r,t),\\
%\forall t>e:\ \Stop(r,t),\\
%r(e)\neq r(s).
%\end{cases}
%\]

% --- Problem predicates ---

% Problem A: both move exactly once and rendezvous
\[
RDV_1\;:\;RDV (\equiv :\ \exists T\ge0\ \forall t\ge T:\ a(t)=b(t))\ \land\ \moves{a}=1\ \land\ \moves{b}=1.
\]

% Problem B: one moves at least twice, the other moves at most once and (the at-most-once mover) eventually stops
\[
AM\;:\;
\Big(\moves{a}\ge 2\ \land\ \moves{b}\le 1\ \land\ \Stopped{b}\Big)
\ \ \lor\ \
\Big(\moves{b}\ge 2\ \land\ \moves{a}\le 1\ \land\ \Stopped{a}\Big).
\]
\Newcodeline
%Lemma:はFCOM^Aで解ける．
%Lemma:CはFSTA^Rで解けない．
%Note AはFSTA（OBLOT）のFとRを分離，LUMI^AやFCOM^Rでとける，FCOMのRとSを分離，BはLUMI^Fで解けない
The problem $RDV_1 \lor AM$ is denoted as $RDAM$.

\begin{algorithm}[H]
\caption{Algorithm for %$RDV_1 \lor AM$  
$RDAM$ in $\FC^{A_{CM}}$\\% (CM-atomic, rigid, no IDs, peer-only lights). 
Each robot $r$ has a visible light $L\in\{\INIT,\ANCHOR\}$ (initially set to $\INIT$).}
\label{alg:FCOM-C}
\small 
\begin{tabbing}
111 \= 11 \= 11 \= 11 \= 11 \= 11 \= 11 \= \kill
%{\em Parameters}: scheduler\crm
%{\em lights}: me.color $\in \{A,B\}$\crm
%{\em Assumptions}: $\OB$, \RSY\crm
%\crm
%{\em State Look}\crm
%\> $r.\pos$, $\other.\pos$: positions of robot $r$ and the other robot;\crm
%Note that $r$ can see only colors of own robot and cannot see the other robots' color in $\mathcal{FSTA}$.\crm
%(Note I am robot $x$)\crm\crm
%{\em Input}:
%\crm

{\em State Compute}\crm
\Cl {\bf if} $peer.L = \INIT$ {\bf then} \crm
\Cl \> $r.L \gets \ANCHOR$\crm
\Cl \> $r.des \gets (p_1 + p_2)/2$ \crm
\Cl {\bf else} Do nothing\crm 
\crm
{\em State Move}\crm
Move to $r.des$;
\end{tabbing}
\end{algorithm}

%\begin{algorithm}[H]
%\caption{Algorithm for $C=A\lor B$ in $\FCOM^{\ASY}$ ($CM$-atomic, rigid, no IDs, peer-only lights)}
%\label{alg:FCOM-C}
%\begin{algorithmic}[1]
%\State Each robot has a visible light $L\in\{\INIT,\ANCHOR\}$ (visible only to the peer).
%\State On activation, the robot takes a snapshot of the peer's position $p_2$ and light $L_2$, and of its own position $p_1$.
%\If{$L_2 = \INIT$}
%    \State Set own light $L \gets \ANCHOR$.
%    \State Move rigidly to the midpoint $m := (p_1 + p_2)/2$.
%\Else \Comment{$L_2 = \ANCHOR$}
%    \State Do nothing (stay stopped forever).
%\EndIf
%\end{algorithmic}%
%\end{algorithm}

\begin{lemma}%[{$C=A\lor B$ is solvable in $\FCOM^{\ASY}$ under CM-atomic, rigid moves, no IDs, and peer-only lights}]
\label{lem:FCOM-CMatomic}
%Let $A$ be “both robots move exactly once and rendezvous,” and let $B$ be
%“one robot moves at most once and eventually stops, while the other robot moves at least twice”
%(rendezvous not required).
%In $\FCOM^{\ASY}$ with CM-atomic Look/Compute--Move (no robot is observable from the start of its
%Compute until the end of its Move), \textsc{rigid} moves, no IDs, and two peer-visible lights
%$\{\INIT,\ANCHOR\}$, 
Algorithm~\ref{alg:FCOM-C} deterministically achieves %$RDV_1 \lor AM$ 
$RDAM$ in $\FC^{A_{CM}}$.
\end{lemma}

\begin{proof}[Proof sketch]
We analyze all possible executions under a fair $CM$-{\bf atomic} \ASY.
(1) both \Look simultaneously (symmetric INIT), (2) one completes its midpoint move before the other’s first \Look, 
(3) one looks while the other is in its $CM$ segment, and (4) both have arrived and are $\ANCHOR$. Due to space limitations, detailed proofs are deferred to Appendix~\ref{app:pos-RDAM}. \qed
\end{proof}

\begin{lemma}[Impossibility of %$RDV_1\lor AM$ 
$RDAM$ in $\FS^{R}$]
\label{lem:impossibility-fsta-rsynch-rigid}
%Consider two anonymous, identical robots executing the same deterministic algorithm in the
%Look--Compute--Move model with visible \emph{finite} colors (FSTA), no IDs, and no external
%asymmetry (no compass, no common coordinates). The scheduler is:
%(i) a finite number of $\FSYNCH$ rounds, followed by
%(ii) infinite $\RSYNCH$ (Round-Robin: one robot activated at a time, alternating).
%Moves are \textsc{rigid} (a commanded destination is always reached).
%Let $A$ denote “each robot moves exactly once and they rendezvous,” and $B$ denote
%“one robot moves at most once and eventually stops, while the other robot moves at least twice”
%%rendezvous not required). Then 
No deterministic algorithm can guarantee $RDAM$ in $\FS^R$.
\end{lemma}

\begin{proof}%[Proof sketch]
\textbf{Initial symmetric configuration.}
Place the robots on a line at symmetric positions
$a(0)=-p$, $b(0)=+p$ with $p\neq 0$, and with the same initial color.

Starting from this initial state, the mutual symmetry between the two robots is maintained throughout the execution (Symmetry preservation lemma), and the number of activations of both robots always remains equal (see Appendix~\ref{app:SPL} for the proof).

\smallskip
\noindent\textbf{Excluding $RDV_1$.}
Property $RDV_1$ requires that both robots rendezvous after exactly one move each.
From the symmetric start, the first \FSY\ move (or the first two \RSY\ activations)
commands mirror-symmetric destinations. With \textsc{rigid} motion, unless the unique
commanded destination is the common midpoint \emph{and} both arrive there, the robots end at
distinct mirror points; by symmetry preservation, no later step can make them co-locate
\emph{without} increasing at least one move count beyond~1, contradicting $RDV_1$.
(If the algorithm tries to force the midpoint, the adversary can schedule \RSY\ so that
one reaches the midpoint while the other, on its turn, recomputes a different mirror destination
from the updated symmetric snapshot; either way $RDV_1$ is not guaranteed.)

\smallskip
\noindent\textbf{Excluding $AM$.}
Property $AM$ requires a strict asymmetry of move counts: one $\le 1$ (and eventually stopped),
the other $\ge 2$. But by the move-count equality above, $\moves{a}=\moves{b}$ holds always.
Since colors are finite and equal on both sides and there are no IDs or external asymmetries,
a deterministic rule cannot decide that only one side should refrain from (or perform extra) moves.
Thus the required asymmetry cannot be enforced; $AM$ is not guaranteed.

\smallskip
Since the adversary can maintain symmetry so as to violate both $RDV_1$ and $AM$, no deterministic
algorithm in $\FS^{R}$  solves $RDAM$.
\end{proof}
\begin{theorem}
\label{th:RDV1-AM}\
%\begin{enumerate}
    %\item  
    $RDAM \not \in \FS^{R}$ and
 %   under rigid-movement and agreement of chirality.
%    \item  
$RDAM  \in \FC^{A}$.
   % under rigid-movement.
%\end{enumerate}
\end{theorem}
%\paragraph{Remark.}
%The impossibility hinges on anonymity, finitely many colors, and Round-Robin activation
%preserving mirror symmetry. Enabling a \emph{comparable external token} (e.g., IDs or an
%externally visible persistent asymmetry) breaks the symmetry and invalidates the argument.

\section{Concluding Remarks}

In this paper, we completely characterize the computational power of robot models for two robots under the assumptions of variable disorientation (VD), chirality, and rigidity.
To address the parts that could not be separated in previous work, we introduce two new problems, \emph{DMSD} and \emph{RDAH}, which exploit the specific properties of two-robot systems and achieve the same separation results as in the general multi-robot case.
We also show that, under \FSY, the three models other than $\OB$ are equivalent. This result provides the first direct-construction proof of an equivalence that had previously been established only via simulation, giving a more constructive and conceptually transparent understanding of the relation between the models.

A remaining open question is whether the same results hold when the above assumptions are relaxed.
In particular, since chirality is unnecessary except for the \textit{SRO} algorithm, it is an interesting problem to determine whether the same hierarchy can be obtained under only VD and rigidity.

\newpage

% \bibliographystyle{ACM-Reference-Format}
\bibliography{referencesorg}

\clearpage

\appendix
{\Large\bf Appendix}

\section{Definition of Epoch}\label{app:epoch}

\begin{definition}
For each robot $r$ and $i\ge1$, let $t_L(r,i)$, $t_C(r,i)$, $t_B(r,i)$, $t_E(r,i)$
be the times of its $i$-th Look, Compute, Move-begin, and Move-end, with
$t_L(r,i) < t_C(r,i) < t_B(r,i) < t_E(r,i) < t_L(r,i{+}1)$.
Define the per-robot Move-end counter
\[
  \mu_r(t)\ :=\ \bigl|\,\{\,i\ge1 \mid t_E(r,i) \le t\,\}\,\bigr|.
\]

%\begin{definition}[Epoch in ASYNCH (Move-end-based)]
Fix a starting time $T_0\ge0$. For $m\ge0$, define
\[
  T^{E}_{m+1}\ :=\ \inf\ \bigl\{\, u>T^{E}_m \ \bigm|\ \forall r:\ \mu_r(u)\ \ge\ \mu_r(T^{E}_m)+1 \,\bigr\}.
\]
If $T^{E}_{m+1}$ is finite, we call $[T^{E}_m,\,T^{E}_{m+1}]$ the $m$-th \emph{epoch}.
Equivalently, an epoch is the minimal time interval in which every robot
completes at least one Move (possibly of zero length) since the previous epoch boundary.
%\end{definition}

%\begin{remark}[Trivial moves]
%If the model executes \emph{no-op} moves when the destination equals the current position,
%we count them in $\mu_r$. If one prefers to count only nontrivial moves, restrict to
%indices $i$ with the displacement $>0$; the results below remain unchanged for
%stopped-configuration predicates.
%\end{remark}

%\begin{lemma}[Relationship to Compute-based epochs]\label{lem:compute-vs-end}
%Let $\kappa_r(t):=\bigl|\{\,i\ge1 \mid t_C(r,i)\le t\,\}\bigr|$.
%Then for every $r$ and $t$,
%%\(
%0 \le \kappa_r(t) - \mu_r(t) \le 1
%\).
%Consequently, the epoch sequences $(T^{C}_m)$ defined by $\kappa_r$ and $(T^{E}_m)$
%defined by $\mu_r$ interleave, and for any monotone predicate evaluated at stopped
%configurations, the asymptotic epoch complexity is the same under either definition.
%\end{lemma}

%\begin{proof}
%Between $t_C(r,i)$ and $t_E(r,i)$ we have $\kappa_r=\mu_r+1$; immediately after $t_E(r,i)$
%they become equal. Summing over robots yields the interleaving of the minima that define
%5$T^{C}_{m+1}$ and $T^{E}_{m+1}$.
%\end{proof}

%\begin{lemma}[Consistency with semi-synchronous epochs]
Under \SSY, %(and $M$-atomic \ASY), 
 one round makes each activated robot complete its Move
within the round; hence one round increases all relevant $\mu_r$ by one.
Therefore the Move-end-based epoch coincides with the classical
“a minimal sequence of activation rounds covering all robots once”.
In \FSY, one round is one epoch.
%\end{lemma}
\end{definition}

\section{Proof of Lemma~\ref{le:Sim-LU-by-FC-3cols}}\label{app:Sim-LU-by-FC-3cols}
%\textcolor{red}{move to Appendix}
%\paragraph*{Correctness.}
\setcounterref{lemmaduplicate}{le:Sim-LU-by-FC-3cols}
\addtocounter{lemmaduplicate}{-1}

\begin{lemmaduplicate}{\em (Reprint of Lemma~\ref{le:Sim-LU-by-FC-3cols} on 
page~\pageref{le:Sim-LU-by-FC-3cols}).}\label{duplicate:Sim-LU-by-FC-3cols}
Any algorithm $\mathcal{A}$ with $k$ colors and running within $T$ epochs in $\LU^R$ can be
simulated by a simulator with $3k^2$ colors and within $3T$ epochs in $\FC^R$.
\end{lemmaduplicate}

\noindent
\textit{(Proof of Lemma~\ref{le:Sim-LU-by-FC-3cols}.)}
We begin by considering the \RR\ case.
\paragraph*{Aligned times and turns.}
Let $t_0 < t_1 < \dots$ be the times when both robots are stopped.
A \emph{turn} is the interval $(t_k,t_{k+1}]$.
By \RR, in each turn exactly one robot is active, except possibly at $k=0$ if an initial \FSY\ step is allowed.
A turn does not necessarily start in phase \textsf{exc}.
We classify turns according to what the active robot $r$ observes at time~$t_k$:
\[
\text{(Decision turn)}\quad \texttt{peer.phase}=\textsf{exc}
\quad\text{vs.}\quad
\text{(Ack turn)}\quad \texttt{peer.phase}\in\{\textsf{rst},\textsf{cpy}\}.
\]

\paragraph*{Key invariants.}
\begin{itemize}
  \item \textbf{I1 (atomic visibility for decision turns).}
        If \texttt{peer.phase}=\textsf{exc} at $t_k$, then during $(t_k,t_{k+1}]$
        the active robot $r$ updates \texttt{r.my.light} \emph{before} moving, and the peer performs no \emph{Look} until $t_{k+1}$.
        Hence the peer’s next \emph{Look} at $t_{k+1}$ sees the new light, as in $\LU^R$.
  \item \textbf{I1$'$ (ack turns are light-preserving).}
        If $\texttt{peer.phase}\in\{\textsf{cpy},\textsf{rst}\}$ at $t_k$, then the active robot only advances the echo/handshake
        (e.g., \textsf{cpy}$\!\to$\textsf{rst} or \textsf{rst}$\!\to$\textsf{exc}) and may perform at most a no-op move to its current position.
        In particular, the visible light configuration and the positions relevant to $\mathcal{A}$ are preserved.
  \item \textbf{I2 (one decision, one move).}
        In each turn there is at most one call to $\mathcal{A}$ and at most one nontrivial rigid move.
        Moreover, a call to $\mathcal{A}$ and a nontrivial move occur \emph{if and only if} it is a decision turn.
        Equivalently: in a decision turn, the active robot calls $\mathcal{A}$ exactly once and performs exactly one nontrivial rigid move; 
        in an ack turn, it makes no call to $\mathcal{A}$ and no nontrivial move.
  \item \textbf{I3 (echo completion and phase progression).}
        A decision by $r$ sets \texttt{r.phase}=\textsf{cpy}. In the next turn, the peer observes \textsf{cpy}, executes the echo
        \texttt{peer.your.light $\leftarrow$ r.my.light}, and moves its phase to \textsf{rst}.
        Thus, starting from any phase that arises in the protocol, a new decision turn occurs again within at most two turns.
\end{itemize}

\paragraph*{Per-turn conformance (case analysis).}
Fix an arbitrary turn $(t_k,t_{k+1}]$.
\begin{itemize}
  \item If it is a \emph{decision turn} (\texttt{peer.phase}=\textsf{exc}), then by I1 and I2 
        the active robot performs exactly one $\LU^R$ step:
        it reads the same snapshot as in $\LU^R$, computes $(L'_r,p'_r)$, publishes $L'_r$, and moves to $p'_r$, while the peer remains idle.
        Hence the post-state (positions and lights) at $t_{k+1}$ coincides with the corresponding $\LU^R$ post-turn state.
  \item If it is an \emph{ack turn} ($\texttt{peer.phase}\in\{\textsf{rst},\textsf{cpy}\}$),
        then by I1$'$ the turn only progresses the echo/handshake and does not change the $\LU^R$-relevant state.
        By I3, a decision turn occurs again within at most two further turns, preserving the visibility order required by $\LU^R$.
\end{itemize}

\paragraph*{\FSY-prefix robustness.}
In \RSY, the scheduler may perform a finite \FSY\ prefix before switching to alternating activation.
We show that the simulator is correct throughout any such prefix and at the switching point.

\begin{lemma}[\FSY\ round correctness]\label{lem:fsynch-round}
Suppose a round is fully synchronous (both robots active). Then:
\begin{itemize}
  \item If the global phase is $(\cpy,\cpy)$, both execute the echo
        $L^{\mathrm{my}}\!\leftarrow\!L^{\mathrm{your}}_{\text{peer}}$ and set $\phase\leftarrow\rst$.
        The echo is idempotent (the lights remain unchanged).
  \item If the global phase is $(\rst,\rst)$, both set $\phase\leftarrow\exc$; lights and positions are unchanged.
  \item If the global phase is $(\exc,\exc)$, both robots take the decision branch simultaneously:
        each runs $\mathcal{A}$ on the same snapshot, publishes its new light, and moves rigidly.
        The post-state coincides with one $\LU^F$ step.
\end{itemize}
Hence any number of consecutive \FSY\ rounds yields a valid $\LU^{F}$ execution (when in $(\exc,\exc)$), 
interleaved with rounds that only update light/phase (when in $(\rst,\rst)$ or $(\cpy,\cpy)$).
\end{lemma}

\begin{lemma}[Switching to alternation]\label{lem:switch}
Let the scheduler switch from \FSY\ to alternating activation immediately after an arbitrary \FSY\ round.
Then the global phase at the switching boundary is one of
$(\exc,\exc)$, $(\cpy,\cpy)$, or $(\rst,\rst)$.
Let $a$ be the robot activated next in the alternating regime. Then:
\begin{itemize}
  \item From $(\exc,\exc)$: the first alternating turn for $a$ is a decision turn, producing $(\cpy,\exc)$; 
        the next turn activates $b$, which completes the echo, restoring the invariant pattern.
          \item From $(\cpy,\cpy)$: within at most two alternating turns the echo/handshake is completed and a decision turn occurs,
        again reaching the standard pattern $(\cpy,\exc)$ followed by the peer’s echo.
  \item From $(\rst,\rst)$: within alternating turns, it moves to $(\exc,\rst)$,
        after which a decision turn occurs and the above case applies.
\end{itemize}
Thus, after finitely many alternating turns, the system reaches a configuration from which the per-turn conformance above applies.
\end{lemma}

\paragraph*{Liveness and conclusion.}
Alternation and rigid moves imply that every turn completes.
By I3, the system returns to phase \textsf{exc} within at most two turns after each decision, so decisions (and hence $\LU^R$ steps) occur infinitely often.
If both robots are initially active at $t_0$, Lemma~\ref{lem:fsynch-round} shows that the initial \FSY\ behavior is consistent with $\LU^F$, 
and Lemma~\ref{lem:switch} shows that the subsequent alternating behavior preserves the simulation.

By induction over the sequence of turns, the simulator’s positions and lights coincide with those of $\mathcal{A}$ in $\LU^R$,
up to the constant overhead in colors and epochs stated above.
Therefore Lemma~\ref{le:Sim-LU-by-FC-3cols} holds.
\qed

\section{Proof of Lemma~\ref{le:Sim-FCAM-by-FCA-12-k-cols}}\label{app:Pr-th5}
\setcounterref{lemmaduplicate}{le:Sim-FCAM-by-FCA-12-k-cols}
\addtocounter{lemmaduplicate}{-1}

\begin{lemmaduplicate}[Reprint of Lemma~\ref{le:Sim-FCAM-by-FCA-12-k-cols} on 
page~\pageref{le:Sim-FCAM-by-FCA-12-k-cols}]\label{duplicateth:Sim-FCAM-by-FCA-12-k-cols}
Any algorithm $\mathcal{A}$ with $k$ colors and within $T$ epochs in $\FC^{A_{CM}}$ can be
simulated by a simulator with $7k$ colors and $4T$ epochs in $\FC^A$.
%The protocol SIM is correct, i.e. any execution of protocol SIM in $\FC_{3k^2}^{RS}$ corresponds to a possible execution of $\LU_{k}^{RS}$.
\end{lemmaduplicate}

\input{sim-fcomAM-by-fcomA}

\section{Shrinking ROtation (SRO)}\label{app:sro}
\paragraph*{\textit{SRO} (Shrinking ROtation)~\cite{FSW19}}

Two robots $a$ and $b$ are  initially  placed in arbitrary distinct points (forming the initial configuration $C_0$).
% forming configuration $C_0$. 
The two robots uniquely identify a square (initially $Q_0$) whose diagonal is given by the segment  between them\footnote{By  square, we means the entire space delimited by the four sides.}. 
Let $a_0$ and $b_0$ indicate the initial positions of the robots, $d_0$  the segment between them, and $\length(d_0)$ its length.
Let $a_i$ and $b_i$ be the  positions of $a$ and $b$ in   configuration $C_i$ ($i \geq 0$). 
The problem consists of moving from  configuration $C_i$  to $C_{i+1}$ in such a way that 
either Condition~1 or Condition~2 is satisfied, and Condition~3 is satisfied in either case:

\begin{itemize}
%\item[Condition~1]  Condition~1:$d_{i+1}$ is a $90$ degree clockwise rotation of $d_i$
\item Condition~1: $d_{i+1}$ is a $90^\circ$ clockwise rotation of $d_i$
% around  $d_i$'s midpoint and thus  $Q_i = Q_{i-1}$, 
and thus  $\length(d_{i+1})= \length(d_i)$,
\item Condition~2: $d_{i+1}$ is a ``shrunk'' $45^\circ$ clockwise rotation of $d_i$ 
%along one endpoint 
such that $\length(d_{i+1}) = \frac{\length(d_i)}{\sqrt{2}}$, 
\item Condition~3: $a_{i+1}$ and $b_{i+1}$ must be contained in the square $Q_{i-1}$\footnote{We define $Q_{-1}$ to be the whole plane.}.
\end{itemize}
%\paragraph*{Intuition.}
Think of the segment between the two robots as a \emph{rotating ruler} and the square it defines as a \emph{picture frame}.
At each step, one of two clockwise operations is performed inside the current frame:
\begin{itemize}
  \item \textbf{Pure rotation:} rotate the ruler by \(90^\circ\) around its midpoint; its length is unchanged.
  \item \textbf{Pivot-and-shrink:} pivot the ruler by \(45^\circ\) about one endpoint while shortening it by a factor \(1/\sqrt{2}\).
\end{itemize}
In both cases, the next robot positions remain inside the previous square \(Q_{i-1}\).
Repeating this yields a clockwise “turn-and-shrink” dynamic: if only \(90^\circ\) steps are taken, the diagonal simply spins at constant length; whenever \(45^\circ\) steps are inserted, the diagonal (and thus the square) shrinks by \((1/\sqrt{2})^k\) after \(k\) such steps, forming a nested sequence of squares and bringing the robots progressively closer—always within the last frame.

\section{Proof of Lemma~\ref{le:dmsd-rigid-asynch}}\label{app:imp-DMSD}

\setcounterref{lemmaduplicate}{le:dmsd-rigid-asynch}
\addtocounter{lemmaduplicate}{-1}

\begin{lemmaduplicate}[Reprint of Lemma~\ref{le:dmsd-rigid-asynch} on 
page~\pageref{le:dmsd-rigid-asynch}]\label{duplicatele:dmsd-rigid-asynch}
%Consider the standard asynchronous Look--Compute--Move model with \textsc{RIGID} moves (each commanded move reaches its destination), anonymous robots, and either \textsc{OBLOT} (no internal memory) or \textsc{FSTA} (finite internal memory not externally visible). 
%For any deterministic algorithm on $\OB$ or $\FS$ that commands a positive move on some instance, there exist an initial configuration and a fair adversarial schedule such that the \textsc{DMSD} property is violated. 
No nontrivial deterministic algorithm solves \textsc{DMSD} in $\FS^A$ or $\OB^A$.
\end{lemmaduplicate}

\begin{proof}[Proof]
Let $D=0:=|b(0)-a(0)|$. Assume the algorithm commands a nonzero move when a robot looks at $(a(0),b(0))$. W.l.o.g.\ robot $a$ is activated first and, by similarity invariance in a two-point scene, chooses a destination on the line $\overline{a(0)b(0)}$ of the form
\[
  A' \;=\; (1-\lambda)a(0)+\lambda b(0) \qquad (\lambda\in(0,1]) .
\]
Robot $a$ starts moving \textsc{rigid}ly towards $A'$. At an adversarially chosen fraction $x\in(0,1)$ of $a$'s progress, robot $b$ is activated and sees $a$ at
\[
  A_x \;=\; a(0) + x(A'-a(0)) \;=\; a(0) + x\lambda v .
\]
By the same invariance, $b$ computes a destination on $\overline{a(0)b(0)}$ with the same coefficient $\lambda$ applied to the seen configuration:
\[
  B' \;=\; (1-\lambda)b(0) + \lambda A_x .
\]
The adversary adjusts speeds so that the two arrivals occur simultaneously (or simply considers the later arrival time); that time is a ``simultaneous stop''. The distance at that time is
\[
  \lVert A'-B'\rVert \;=\; \bigl\lVert \bigl((1-\lambda)-\lambda+\lambda^2 x\bigr) v \bigr\rVert
  \;=\; \bigl|\,1-2\lambda+\lambda^2 x\,\bigr|\,D_0 .
\]
For any fixed $\lambda\in(0,1]$, the map $x\mapsto |1-2\lambda+\lambda^2 x|$ is a non-constant affine function on $(0,1)$, hence its image contains a nondegenerate interval. Since $\Dyad$ is countable and has empty interior, the adversary can pick $x$ so that
\[
  \frac{\lVert A'-B'\rVert}{D_0} \;=\; |1-2\lambda+\lambda^2 x| \;\notin\; \Dyad ,
\]
violating \textsc{DMSD} at a simultaneous stop. The only way to avoid this is the degenerate ``never move'' policy, which we exclude as trivial.
\end{proof}

\section{Definition of Single Move(SM)}\label{app:SM}

\begin{definition} 

% \textcolor{blue}
% {{\bf Problem Single Move Line Shrink} (SMLS):
% Let $a,b,$ be two robots on distinct locations $a(0), b(0)$
% where $r(t)$ denotes the position of  $r\in\{a,b\}$ at time $t\geq 0$.
% The SMLS  problem requires the robots to move closer to each other  performing a single move each.
% }

{\bf Single Move} (SM):
Let $a,b,$ be two robots on distinct locations $a(0), b(0)$
where $r(t)$ denotes the position of  $r\in\{a,b\}$ at time $t\geq 0$.
%The problem SM requires each of the two robots to make exactly one non-zero movement.
The problem SM requires each of the two robots to move only once to a different location.
%\end{definition}
%In other words, 
More precisely, an algorithm solves SM iff it satisfies the following temporal geometric predicate:

% $$ LS \equiv \Big[\exists t_2\geq t_1> 0 :
% \big\langle dis(a(t_2),b(t_2))\leq  dis(a(t_1),b(t_1)) <   dis(a(0),b(0)) \big\rangle \wedge
% \big\langle \forall r\in\{a,b\}, $$
%   $$ \big\{ {\bf if}\  dis(r(0),r(t_1)) \neq 0\   {\bf then}\
%               \big( \forall t' < t_1  < t" : r(t')=r(0) \wedge r(t")=r(t_1)\big)\  {\bf else}$$                            
%   $$    \big(   ( dis(r(0),r(t_2)) \neq 0)  \wedge   ( \forall t' < t_2  < t" : r(t')=r(0) \wedge r(t")=r(t_2)  )
%      \big)   \}   \big\rangle    \Big]$$
%\color{blue}
%Is the following correct?
%$$ SMLS \equiv \Big[\exists! %t_a\exists!\ t_b\ s.t.\ %t_1=max(t_a,t_b)\geq  %t_2=\min(t_a,t_b)> 0 :$$
%$$\big\langle dis(a(t_1),b(t_1))\leq  dis(a(t_2),b(t_2)) <   dis(a(0),b(0)) \big\rangle \wedge $$
%  $$ \big\{  (dis(a(0),a(t_a)) \neq 0)\wedge                 \big( \exists t' < t_a  < \forall t" : a(t')=a(0) \wedge a(t")=a(t_a)\big)\  \wedge $$                           % 
%  $$    \big(   ( dis(b(0),b(t_b)) \neq 0)  \wedge   ( \exists t' < t_b  < \forall t" : b(t')=b(0) \wedge b(t")=b(t_b)  )
%     \big)   \}   \big\rangle    \Big]$$
%\color{black}
%\color{red}
$$
SM \equiv \moves{a}=1\ \land\ \moves{b}=1.
%\forall r \in \{a,b\},  \exists t_1,t_2:
%\left (
%\begin{aligned}
%&r(t_1) \neq r(t_2) \land (0 \leq \forall t \le t_1: r(t) = r(t_1))
%\land (\forall t \ge t_2: r(t_2)=r(t))
%\\
%&\ \  \land (\forall t,t': t_1 \le t < t' \le t_2 \to r(t) \neq r(t'))
%\end{aligned}
%\right )
$$\qed

%$$
%SM \equiv
%\forall r \in \{a,b\},  \exists t_1,t_2:
%\left (
%\begin{aligned}
%&r(t_1) \neq r(t_2) \land (0 \leq \forall t \le t_1: r(t) = r(t_1))
%\land (\forall t \ge t_2: r(t_2)=r(t))
%\\
%&\ \  \land (\forall t,t': t_1 \le t < t' \le t_2 \to r(t) \neq r(t'))
%\end{aligned}
%\right )
%$$
\end{definition}

\section{Proof of Lemma~\ref{lem:FCOM-CMatomic}}\label{app:pos-RDAM}

\setcounterref{lemmaduplicate}{lem:FCOM-CMatomic}
\addtocounter{lemmaduplicate}{-1}

\begin{lemmaduplicate}[Reprint of Lemma~\ref{lem:FCOM-CMatomic} on 
page~\pageref{lem:FCOM-CMatomic}]%[{$C=A\lor B$ is solvable in $\FCOM^{\ASY}$ under CM-atomic, rigid moves, no IDs, and peer-only lights}]
\label{duplicatelem:FCOM-CMatomic}
%Let $A$ be “both robots move exactly once and rendezvous,” and let $B$ be
%“one robot moves at most once and eventually stops, while the other robot moves at least twice”
%(rendezvous not required).
%In $\FCOM^{\ASY}$ with CM-atomic Look/Compute--Move (no robot is observable from the start of its
%Compute until the end of its Move), \textsc{rigid} moves, no IDs, and two peer-visible lights
%$\{\INIT,\ANCHOR\}$, 
Algorithm~\ref{alg:FCOM-C} deterministically achieves %$RDV_1 \lor AM$ 
$RDAM$.
\end{lemmaduplicate}

\begin{proof}
We analyze all possible executions under a fair $CM$-{\bf atomic} \ASY.
(1) both look simultaneously (symmetric INIT), (2) one completes its midpoint move before the other’s first Look, 
(3) one looks while the other is in its $CM$ segment, and (4) both have arrived and are $\ANCHOR$. 

Let $m$ be $(p_1+p_2)/2$.
\paragraph{Case~1: both look simultaneously (symmetric INIT).}
Both robots see $L=\INIT$, apply the same rule, and compute the same midpoint $m$.
Because of CM-atomicity, neither observes the other’s motion.
Each moves rigidly to $m$ and stops there after exactly one move.
Hence $RDV_1$ holds (both move once and rendezvous).

\paragraph{Case~2: one completes its midpoint move before the other’s first Look.}
Suppose robot $a$ performs the first Look, sets $a.L=\ANCHOR$, and moves rigidly to $m$.
Robot $b$ performs its first Look after $a$ has reached $m$, so it observes $a.L=\ANCHOR$
and executes the “do nothing” rule.
Robot $b$ thus remains at its initial position forever ($\mathrm{moves}(b)=0$).
By fairness, $a$ is activated again; it now sees $b.L=\INIT$ and therefore applies the midpoint
rule once more, moving to the midpoint between $m$ and $b$’s position.
Hence $\mathrm{moves}(a)\ge2$, $\mathrm{moves}(b)\le1$, and $b$ is eventually stopped.
Therefore $AM$ holds.

\paragraph{Case~3: one looks while the other is in its $CM$ segment.}
Assume $b$ looks while $a$ is executing its Compute--Move for the midpoint rule.
By CM-atomicity, $b$’s snapshot is either:
\begin{enumerate}%[label=(\roman*)]
    \item pre-move: $a$ still appears with $a.L=\INIT$, so $b$ also applies the midpoint rule,
          which reduces to Case~1; or
    \item post-move: $a$ has arrived and displays $a.L=\ANCHOR$, so $b$ applies the “do nothing” rule,
          reducing to Case~2.
    \item \emph{symmetric stale-snapshot interleaving:} after the first robot reaches $m$ and becomes
          $\ANCHOR$, the second robot starts its first midpoint move; if the first robot looks again
          while the second is still in its CM segment, CM-atomicity yields a stale snapshot
          ($L_{\text{peer}}=\INIT$ at the old location), so it applies the midpoint rule once more.
          Then the second completes its first move to $m$ and anchors.
          At the first simultaneous-stop time we have $\mathrm{moves}$ asymmetry, so $AM$ holds.

\end{enumerate}
In either subcase, $RDV_1$ or $AM$ is satisfied.

\paragraph{Case~4: both have arrived and are $\ANCHOR$.}
Once both are $\ANCHOR$, any subsequent Looks see $L=\ANCHOR$,
and both remain stopped forever.
No “swap” or further motion occurs.
Hence once $RDV_1$ or $AM$ has been realized, it remains true thereafter.

\medskip
These four cases exhaust all adversarial CM-atomic schedules.
In every possible fair execution, the outcome satisfies either $A$ or $B$.
Hence Algorithm~\ref{alg:FCOM-C} solves $RDV_1\lor AM$ deterministically in the stated model. \qed
\end{proof}

\section{Proof of Symmetry preseving lemma}\label{app:SPL}

\smallskip
\noindent\textbf{Symmetry preservation lemma.}
From this initial condition, an adversarial scheduler can preserve,
after every Look/Compute/Move,
\[
a(t)=-b(t)\qquad\text{and}\qquad \mathrm{color}_a(t)=\mathrm{color}_b(t).
\]
Indeed, during the (finite) \FSY\ prefix, both robots see identical (mirror) snapshots,
perform identical computations, and command mirror destinations; \textsc{rigid} motion returns them
to a mirror-symmetric configuration. In the subsequent \RSY\ phase, when robot say $a$ is activated
alone and moves to a destination $f(\text{view},\text{color})$, the next activation of robot $b$
occurs from the mirrored view with the same color, hence robot~$b$ commands the mirror destination;
after its \textsc{rigid} move, mirror symmetry (and equal colors) is restored. By induction, symmetry
is preserved forever.

\smallskip
\noindent\textbf{Corollary (move-count equality).}
Let $\moves{r}$ denote the number of commanded moves completed by robot $r$ up to any time.
Under \RR with preserved symmetry, at each alternation either both increase their
move count by one (over two consecutive activations) or neither does; hence
\[
\moves{a}=\moves{b}\quad\text{at all times.}
\]

%We now prove that Protocol SIM provides a fair and correct execution of any \textit{restricted-repetition} protocol ${\cal{P}}$.

% \section{Pseudocode of \texorpdfstring{$\rootopt$}{RootOptimal}}
% \label{sec:pseudocode_root_optimal}
% The pseudocode of $\rootopt$ is shown in Algorithm \ref{al:root_optimal}.

\end{document}

%% file: sim-abst.tex
\vspace{-0.3cm}

\paragraph*{Overview of the simulator {\tt SIM($\mathcal{A}$)} (for two robots)}

We propose a new simulation protocol {\tt SIM($\mathcal{A}$)} that simulates any two-robot algorithm
$\mathcal{A}$ designed for the $CM$-{\bf atomic} model $\FC^{A_{CM}}$
within the plain asynchronous model $\FC^{A}$.
Unlike the previous general-$n$ construction~\cite{FSSW23}, this simulator exploits the
two-robot symmetry to minimize both the number of epochs and the number of colors.

Each robot maintains three public lights
(\texttt{light}, \texttt{phase}, \texttt{my-state} and \texttt{your-state})
to coordinate a short three-phase handshake:
\textsf{exc} (execute), \textsf{cpy} (copy), and \textsf{rst} (reset).
This handshake serves the same function as the algorithm described in Section~\ref{sec:LUReqFCR}.
Furthermore, to ensure in the \ASY\ setting that each robot’s \emph{Compute--Move} interval remains invisible to its peer, we introduce  control variables \texttt{my-state} and \texttt{your-state}.
After completion of exection of $\mathcal{A}$, a robot sets \texttt{my-state} to \textsf{M}, indicating that it is currently executing the algorithm.
If the peer observes \textsf{M} during a \emph{Look}, it refrains from executing $\mathcal{A}$.
In the \textsf{cpy} phase, each robot $r$ copies the peer’s \texttt{my-state} to \texttt{r.your-state}
so that it can correctly track the peer’s execution status\footnote{Because {\tt SIM($\mathcal{A}$)} simulates $\FC^{A_{CM}}$ under $\FC^{A}$, it is unnecessary to copy the lights used by $\mathcal{A}$}.

%This handshake guarantees that each robot’s
%\emph{Compute--Move} interval is invisible to the peer,
%realizing the same atomicity as in $\FC^{A_{CM}}$.
A complete simulation cycle consists of at most four epochs,
during which one atomic step of $\mathcal{A}$ is faithfully reproduced.
Thus, each call to $\mathcal{A}$ corresponds to one atomic move in $\FC^{A_M}$,
preserving both correctness and fairness.

The key improvement lies in the reduced color space:
although the simulator uses $3\times2\times2=12$ colors,
only seven effective patterns appear in the transition diagram
(See Fig.~\ref{fig:trans} in Appendix~\ref{app:Pr-th5}),
resulting in a $7k$-color simulator when $\mathcal{A}$ uses $k$ colors.
This represents a significant simplification over previous multi-robot
simulators, both in state complexity and in phase length.
All invariants and the full correctness proof of {\tt SIM($\mathcal{A}$)}
(including $CM$-atomic safety and cycle refinement mapping)
are provided in Appendix~\ref{app:Pr-th5}.

%% file: sim-fcomAM-by-fcomA.tex
\begin{algorithm}[t]
\caption{{\tt SIM($\mathcal{A}$)} - for robot $r$}% at location $x$}
\label{algo:simA-A-CM}
{\small
\begin{tabbing}
111 \= 11 \= 11 \= 11 \= 11 \= 11 \= 11 \= \kill

{\em State Compute}\crm
\Cl \> r.des $\leftarrow$ r.pos  \crm
\Cl \> {\bf case} other.phase {\bf of}\crm
\Cl \>exc: {\bf case} (peer.your-state, peer.my-state) {\bf of} \crm
\Cl \>\>(W,W),\crm
\Cl \> \>(W,M): \;Execute the {\tt Compute} of $\mathcal{A}$ \`// determining my color $r.light$ and destination $r.des$ // \crm
\Cl \>\>\>\>\>r.phase $\leftarrow$ cpy\crm
\Cl \>\>\>\>\>r.my-state $\leftarrow$ M\crm
\Cl \>\>\>\>\>r.your-state $\leftarrow$ peer.my-state\crm
\Cl \>\>(M,W):\;r.phase $\leftarrow$ exc\crm
\Cl \>\>\>\>\>r.your-state $\leftarrow$ peer.my-state\crm
\Cl \>\>(M,M): \;r.phase $\leftarrow$ cpy\crm
\Cl \>\>{\bf end-case}\crm

\Cl \>cpy: {\bf case} (peer.your-state, peer.my-state) {\bf of} \crm
\Cl \>\>(M,W): \;r.phase $\leftarrow$ exc\crm
\Cl \>\>\>\>\>r.your-state $\leftarrow$ peer.my-state\crm
\Cl \>\>(W,M): \;r.phase $\leftarrow$ cpy\crm
\Cl \>\>\>\>\>r.your-state $\leftarrow$ peer.my-state\crm
\Cl \>\>(M,M): \;r.phase $\leftarrow$ rst\crm
\Cl \>\>\>\>\>r.your-state $\leftarrow$ W\crm
\Cl \>\>\>\>\>r.my-state $\leftarrow$ W\crm
\Cl \>\>{\bf end-case}\crm

\Cl \>rst: \;r.phase $\leftarrow$ exc\crm
\Cl \>\>r.your-state $\leftarrow$ W\crm
\Cl \>\>r.my-state $\leftarrow$ W\crm
\Cl \> {\bf end-case}\crm

{\em State Move}\crm
Move to $r.des$;\crm
\end{tabbing}
}
\end{algorithm}

% =========================
% Overview (body text)
% =========================
%\textcolor{red}{must be shorted and only result in main page}
\paragraph*{Behavior of the simulator {\tt SIM}($\mathcal{A}$) (for two robots)}

We simulate any two-robot algorithm $\mathcal{A}$ designed for the $CM$-{\bf atomic} model
$\FC^{A_{CM}}$ inside the plain asynchronous model $\FC^{A}$ by means of a short, finite
handshake driven by \emph{public} variables.

Each robot $r\in\{a,b\}$ maintains three public lights:
\begin{itemize}
  \item \textit{r.light} — the light value that $\mathcal{A}$ exposes for $r$,
  \item $\textit{r.phase} \in \{\textsf{exc},\textsf{cpy},\textsf{rst}\}$ \ (handshake progress),
  \item $\textit{r.my-state} \in \{W,M\}$ \ (whether $r$ has already decided in the current cycle),
  \item $\textit{r.your-state} \in \{W,M\}$ \ ($r$’s last copy of the peer’s $\textit{my-state}$).
\end{itemize}
At every {\em Compute}, we first set $r.\textit{des}\gets r.\textit{pos}$; if a call to
$\mathcal{A}$ occurs, it may then update $r.\textit{des}$.

%\paragraph*{Branching rule (crucial).}
The simulation protocol {\tt SIM}($\mathcal{A}$) is described in Algorithm~\ref{algo:simA-A-CM}.
The \emph{outer} case distinction is on the \emph{peer’s} phase ($\textit{peer.phase}$):
\textsf{exc} (decision stage), \textsf{cpy} (copy/wait), or \textsf{rst} (reset).
When $\textit{peer.phase}=\textsf{exc}$, an \emph{inner} case on
$(\textit{peer.your-state},\textit{peer.my-state})\in\{W,M\}^2$ is used.

%\paragraph*{Effect.}
\begin{itemize}
  \item In the decision stage (\textsf{exc}), seeing $(W,W)$ or $(W,M)$ from the peer makes $r$
        call $\mathcal{A}$, then set $\textit{r.my-state}$ to $M$ and move to \textsf{cpy}.
        From a symmetric start, one or both robots may thus call $\mathcal{A}$, which is admissible in $\FC^{A_M}$.
  \item In the copy stage (\textsf{cpy}), no calls to $\mathcal{A}$ are made; the robots only
        propagate flags until both see $(M,M)$, then both enter \textsf{rst} and reset to $(W,W)$.
  \item In the reset stage (\textsf{rst}), both immediately re-align to $(\textsf{exc},W,W)$,
        opening the next cycle.
\end{itemize}

%\paragraph*{Asymmetry under peer inactivity.}
Because the \emph{outer} branch keys on the \emph{peer’s} phase, even if $r$ is locally in \textsf{cpy},
as long as the peer remains in \textsf{exc}, $r$ executes the \textsf{exc}-branch at its next activation.
Hence, while the peer is inactive, $r$ may call $\mathcal{A}$ \emph{repeatedly}, producing several
consecutive moves; in $\FC^{A_M}$ this corresponds to several consecutive \emph{atomic} steps of $r$.

The transition of configurations is shown in Fig.~\ref{fig:trans}.
\medskip
\begin{figure}
%    \centering
%    \includegraphics{}
\centering\includegraphics[keepaspectratio, width=14cm]{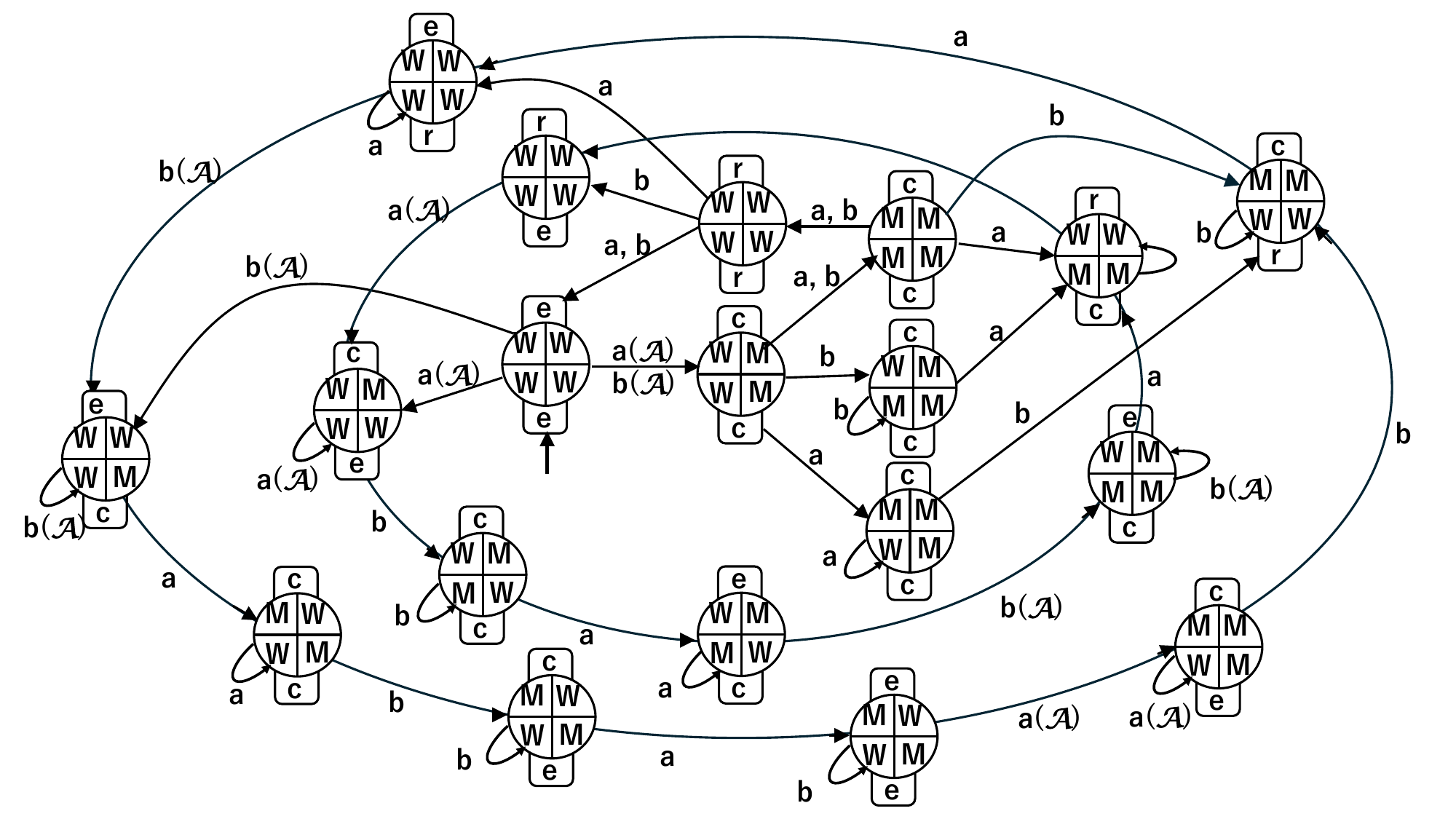}
\centering\includegraphics[trim=6cm 12cm 0cm 1.5cm, clip,width=12cm]{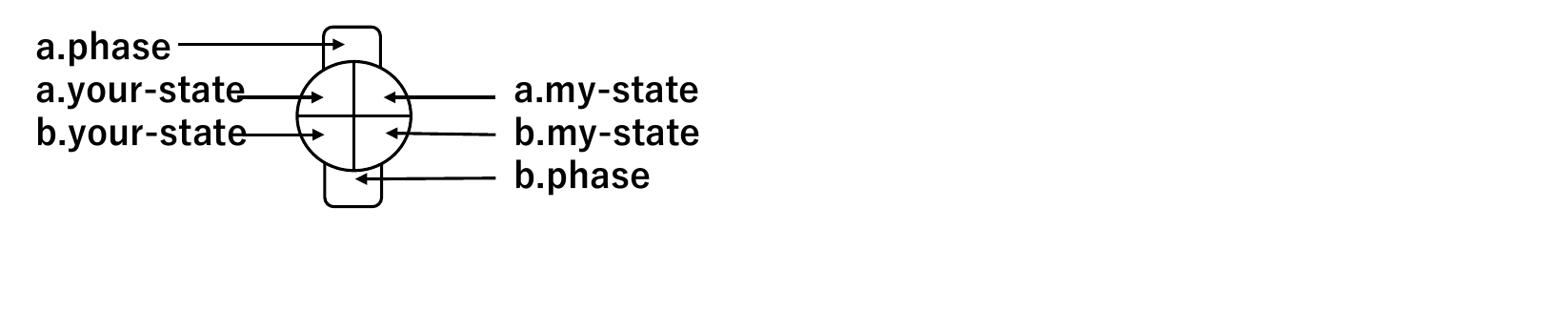}
    \caption{Transition Diagram of {\tt SIM}$(\mathcal{A})$. Each node encodes the colored configuration of the two robots, where the upper (resp. lower) semicircle and the symbol above (resp. below) it represent the state of robot $a$ (resp. $b$). The symbols \textsf{e}, \textsf{c}  and \textsf{r} denote \textsf{exc}, \textsf{cpy}, and \textsf{rst}, respectively. Directed edges correspond to possible moves between nodes. 
An edge labeled $r \in \{q,b\}$
indicates that robot 
 performs an action at that transition,
whereas a label $r(\mathcal{A}) (r \in \{q,b\})$
 means that robot 
$r$ executes one step of the simulated algorithm $\mathcal{A}$.
These labeled edges specify how the pair of robots progresses through the state space during the simulation.}
    \label{fig:trans}
\end{figure}
% =========================
% Correctness sketch (body text)
% =========================

% =========================
% Correctness sketch (revised with explicit Move-atomic guarantee)
% =========================
\paragraph*{Correctness}

%\paragraph*{Move-atomic guarantee.}
In {\tt SIM}($\mathcal{A}$), each robot’s behavior from the start of its {\em Compute} phase until the completion of its {\em Move} is intentionally hidden from the peer.
Specifically, during the {\em Compute} phase, the robot sets \texttt{my-state} to \textsf{M}.
When the peer observes \textsf{M}, it refrains from executing $\mathcal{A}$ until the state returns to normal, instead entering a mode where it copies the peer’s state.
During the \textsf{cpy} phase, each robot copies the peer’s \texttt{my-state}, allowing it to correctly infer the peer’s current execution status.
%In {\tt SIM}($\mathcal{A}$), each robot’s behavior from the start of its {\em Compute}
%phase until the completion of its {\em Move} is intentionally hidden from the peer:
%during this interval, the peer’s Look operation perceives only the robot’s
%\emph{previous} stable state (position and lights), never an intermediate position.
This ensures that every call to $\mathcal{A}$ is observed by the other robot as
a single, indivisible event—exactly as in the $CM$-atomic model.
Consequently, the continuous motion generated by a call to $\mathcal{A}$
can be treated as one \emph{atomic move} with respect to the peer’s view.

%\paragraph*{Cycles.}
We say that a configuration is \emph{aligned} when both robots are in state
$(\textsf{exc},W,W)$ or one robot is state $(\textsf{exc},W,W)$ and another one is in $(\textsf{rst},W,W)$.  
A \emph{cycle} begins at an aligned configuration and ends when the next aligned
configuration is reached.
Within one cycle, the following invariants hold:
\begin{itemize}
  \item \textbf{(I-M) Flag monotonicity:}  
        $\textit{my-state}$ changes only from $W\!\to\!M$ (it returns to $W$ only via
        the reset stage).  
        $\textit{phase}$ progresses monotonically
        $\textsf{exc}\!\to\!\textsf{cpy}\!\to\!\textsf{rst}\!\to\!\textsf{exc}$.
  \item \textbf{(I-Q) Call uniqueness:}  
        Each robot calls $\mathcal{A}$ at most once per cycle.
        After a call occurs, no further calls appear until the reset stage.
  \item \textbf{(I-D) Stable destinations:}  
        $\textit{des}$ is updated only at a call to $\mathcal{A}$; otherwise it remains fixed.
  \item \textbf{(I-Mov) One rigid move per call:}  
        Every call to $\mathcal{A}$ triggers exactly one continuous rigid move to
        $\textit{des}$, after which the robot stays still until the end of the cycle.
\end{itemize}

\paragraph*{Safety (atomic-step simulation).}
By (I-Q), at most one new decision is made per robot per cycle.
If both robots call $\mathcal{A}$ in the same cycle, their actions correspond to
two atomic moves occurring within the same abstract step of $\FC^{A_{CM}}$.
If only one calls, the other remains stationary during that step.
Since no robot observes another’s in-progress move (by the Move-atomic guarantee),
the pair of continuous moves within a cycle constitutes one atomic transition
of $\mathcal{A}$.

\paragraph*{Liveness.}
From $(\textsf{exc},W,W)$, some robot inevitably calls $\mathcal{A}$ on
snapshot $(W,W)$ or $(W,M)$, initiating the cycle.
Through the copy phase (\textsf{cpy}), both eventually observe $(M,M)$,
then simultaneously enter \textsf{rst} and reset to $(\textsf{exc},W,W)$.
Under weak fairness, every cycle completes in finite time, and infinitely
many cycles occur (no deadlock).

\paragraph*{Refinement mapping to $\FC^{A_{CM}}$.}
Each cycle corresponds to one abstract atomic step of $\mathcal{A}$ in $\FC^{A_M}$.
Robots that called $\mathcal{A}$ in the cycle perform exactly one atomic move
to their computed destination; others remain stationary.
By (I-D) and (I-Mov), continuous motions are collapsed into single atomic transitions,
and aligned configurations coincide with post-states of the abstract steps.
If one robot is inactive for several consecutive cycles, its peer’s repeated calls
map to multiple consecutive abstract steps of the same robot.
Thus, the concrete execution under {\tt SIM}($\mathcal{A}$) refines the abstract
execution of $\mathcal{A}$ in $\FC^{A_{CM}}$.

From Fig.~\ref{fig:trans}, one execution epoch of algorithm $\mathcal{A}$ by robots $a$ and $b$ is simulated by four epochs in the simulator.
The number of colors used in the simulation is $3\times2\times2=12$; however, as shown in Fig.~\ref{fig:trans}, only seven distinct patterns actually appear in the transition diagram.
Therefore, this simulation requires only $7k$ colors in total, where $k$ is the number of colors used by algorithm $\mathcal{A}$.